\documentclass{article}
\usepackage{authblk}

\usepackage[utf8]{inputenc} 
\usepackage[T1]{fontenc}    
\usepackage{url}            
\usepackage{booktabs}       
\usepackage{amsfonts}       
\usepackage{nicefrac}       
\usepackage{microtype}      

\usepackage{amsmath}
\usepackage{amsthm}
\usepackage{amssymb}
\usepackage{algorithm,algpseudocode}

\usepackage{microtype}
\usepackage{graphicx}
\usepackage{subfigure}
\usepackage{wrapfig}
\usepackage{booktabs} 
\usepackage{enumitem}

\usepackage{tikz}
\usepackage{pgfplots}
\usepackage{pdflscape}
\pgfplotsset{compat=newest}
\usetikzlibrary{plotmarks}
\usetikzlibrary{arrows.meta}
\usepgfplotslibrary{patchplots}
\usepackage{grffile}
\newlength\figurewidth

\usepackage{amsmath}
\usepackage{amssymb}
\usepackage{mathtools}
\usepackage{amsthm}

\newtheorem{theorem}{Theorem}[section]

\newtheorem{lemma}[theorem]{Lemma}
\newtheorem{corollary}[theorem]{Corollary}
\newtheorem{claim}[theorem]{Claim}
\theoremstyle{definition}
\newtheorem{definition}[theorem]{Definition}
\newtheorem{assumption}[theorem]{Assumption}
\theoremstyle{remark}

\usepackage[textsize=tiny]{todonotes}

\usepackage{xr-hyper}
\makeatletter
\newcommand*{\addFileDependency}[1]{
  \typeout{(#1)}
  \@addtofilelist{#1}
  \IfFileExists{#1}{}{\typeout{No file #1.}}
}
\makeatother

\usepackage{hyperref}
\usepackage[capitalize,noabbrev]{cleveref}

\newcommand{\tg}[1]{\widetilde{g}_{#1}}

\usepackage{xcolor}
\definecolor{ourblue}{rgb}{0.368,0.507,0.71}
\definecolor{ourorange}{rgb}{0.881,0.611,0.142}
\definecolor{ourgreen}{rgb}{0.56,0.692,0.195}
\definecolor{ourred}{rgb}{0.923,0.386,0.209}

\hypersetup{
	colorlinks=true,       
	linkcolor=ourblue,        
	citecolor=ourblue,        
	filecolor=ourblue,     
	urlcolor=ourblue
}

\DeclareMathOperator*{\argmin}{arg\,min}   
\DeclareMathOperator*{\argmax}{arg\,max}
\newcommand{\T}{\ensuremath{\top}}                

\usepackage[final]{neurips_2023}

\title{Online Learning under\\ Adversarial Nonlinear Constraints}

\author[1]{\textbf{Pavel Kolev}}
\author[1,2]{\textbf{Georg Martius}}
\author[1]{\textbf{Michael Muehlebach}}
\affil[1]{Max Planck Institute for Intelligent Systems, Tübingen, Germany}
\affil[2]{University of Tübingen, Tübingen, Germany}
\affil[ ]{\texttt{\{pavel.kolev, georg.martius, michael.muehlebach\}@tuebingen.mpg.de}}

\begin{document}

\maketitle

\begin{abstract}
In many applications, learning systems are required to process continuous non-stationary data streams.
We study this problem in an online learning framework and propose an algorithm that can deal with adversarial time-varying and nonlinear constraints.
As we show in our work, the algorithm called Constraint Violation Velocity Projection (CVV-Pro) achieves $\sqrt{T}$ regret and converges to the feasible set at a rate of $1/\sqrt{T}$, despite the fact that the feasible set is slowly time-varying and a priori unknown to the learner. 
CVV-Pro only relies on local sparse linear approximations of the feasible set and therefore avoids optimizing over the entire set at each iteration, which is in sharp contrast to projected gradients or Frank-Wolfe methods. 
We also empirically evaluate our algorithm on two-player games, where the players are subjected to a shared constraint.
\end{abstract}

\section{Introduction}\label{sec:Introduction}
Today's machine learning systems are able to combine computation, data, and algorithms at unprecedented scales, which opens up new and exciting avenues in many domains, such as computer vision, computer graphics, speech and text recognition, and robotics \citep{MitchellMJ}. One of the leading principles that has enabled this progress is the focus on relatively simple pattern recognition and empirical risk minimization approaches, which mostly rely on offline gradient-based optimization and stipulate that the training, validation, and test data are independent and identically distributed.

Somewhat overlooked in these developments is the role of non-stationarity and constraints \citep{JordanRevolution}. 
Indeed, emerging machine learning problems involve decision-making in the real world, which typically includes interactions with physical, social, or biological systems.
These systems are not only time varying and affected by past interactions, but their behavior is often characterized via fundamental constraints. 
Examples include cyber-physical systems where constraints are imposed by the laws of physics, multi-agent systems that are subjected to a shared resource constraint, or a reinforcement learning agent that is subjected to safety and reliability constraints.
In particular, in their seminal work~\citet{Auer02} gave a reduction for the multi-arm bandit setting to the full information online optimization setting, by employing the multiplicative weights framework~\citep{LittlestoneW94}.
This classical reduction was recently extended by~\citet{T4} to the contextual bandit setting with sequential (time-varying) risk constraints.

This motivates our work, which is in line with a recent trend in the machine learning community towards online learning, adaptive decision-making, and online optimization. 
More precisely, we study an online problem with slowly time-varying constraints, governed by the following interaction protocol (see Assumption~\ref{as:ip}).
In each time step $t$, the learner commits a decision $x_t$ and then in addition to a loss value $f_t(x_t)$ with its gradient $\nabla f_t(x_t)$ receives partial information about the current feasible set $\mathcal{C}_t:=\{x\in \mathbb{R}^n~|~g_t(x)\geq0\}$, where the constraint function $g_t(x)$ is defined as $[g_{t,1}(x),\dots,g_{t,m}(x)]$.
The quality of the learner's decision making is measured, for every $T\geq1$, by comparing to the best decision in hindsight $x_{T}^{\star}\in\argmin_{x \in \mathcal{C}_T} \sum_{t=1}^{T} f_t(x)$, that is,
\begin{equation}
\sum_{t=1}^{T} f_t(x_t) - \sum_{t=1}^{T} f_t(x_{T}^{\star}) \quad \text{subject to}\quad g_T(x_T)\geq - \frac{c}{\sqrt{T}}, \label{eq:regret}
\end{equation}
which will be shown to be bounded by $\mathcal{O}(\sqrt{T})$ for our algorithm. 
The functions $f_t$ and $g_t$ are restricted to $f_t\in \mathcal{F}$ and $g_t\in \mathcal{G}$ (as defined in Assumption~\ref{as:fg}) and $c>0$ is an explicit constant. 

It is important to note that our performance objective \eqref{eq:regret} is symmetric in the sense that the constraint $x\in \mathcal{C}_T$ applies to both the learner's decision $x_T$ and the benchmark $x_{T}^{\star}$.
This contrasts prior work by \citet{T2,T3,T4,T1,CaoL19} and \citet{LiuWHF22}, where a different notion of constraint violation $\sum_{t=1}^{T} g_t(x_t)\geq -c_0\sqrt{T}$ is used for the learner, while either a single benchmark $x_{1:T}^{\star}$ satisfies $g_t(x_{1:T}^{\star})\geq 0$ for all $t\in \{1,\dots,T\}$ or multiple benchmarks $\{x_{t}^{\prime}\}_{t=1}^{T}$ satisfy $x_{t}^{\prime}\in\argmin_{x\in\mathcal{C}_t}f_t(x)$.
Unlike \eqref{eq:regret}, different requirements are imposed on the learner and the benchmark(s), which leads to asymmetric regret formulations: $\sum_{t=1}^{T} f_t(x_t) - \sum_{t=1}^{T} f_t(x_{1:T}^{\star})$ and $\sum_{t=1}^{T} f_t(x_t)- \sum_{t=1}^{T} f_t(x_{t}^{\prime})$, respectively.
Furthermore, as our bound $g_T(x_T)\geq - c/\sqrt{T}$ applies for all $T\geq1$, it implies the cumulative constraint violation bound in \citet{T2} up to a constant factor $\sum_{t=1}^{T}g_{t}(x_{t})\geq-c\sum_{t=1}^{T}1/\sqrt{t}\geq-2c\sqrt{T}$.

Even more intriguing is the fact that our algorithm is unaware of the feasible sets a-priori, and obtains, at each iteration, only a local sparse approximation of $\mathcal{C}_t$ based on the first-order information of the \textit{violated} constraints. 
The indices of all violated constraints at $x_t$ will be captured by the index set $I(x_t):=\{i\in \{1,\dots,m\}~|~g_{t,i}(x_t)\leq 0\}$, while $G(x_t):=[\nabla g_{t,i}(x_t)]_{i\in I(x_t)}$ denotes the matrix whose columns store the corresponding gradients. 
In order to guarantee a regret of $\mathcal{O}(\sqrt{T})$ in \eqref{eq:regret} we require the following assumptions.
\begin{assumption}\label{as:fg}
There exist $R,L_\mathcal{F},L_\mathcal{G}>0$: \textbf{1)} $\mathcal{F}$ is a class of convex functions, where every $f\in \mathcal{F}$ satisfies $||\nabla f(x)||\leq L_\mathcal{F}, \forall x\in \mathcal{B}_{4R}$, with $||\cdot||$ the $\ell_2$ norm and $\mathcal{B}_R$ the hypersphere of radius $R$ centered at the origin; \textbf{2)} $\mathcal{G}$ is a class of concave $\beta_\mathcal{G}$-smooth functions, where every $g$ satisfies $||\nabla g(x)||\leq L_\mathcal{G}, \forall x\in \mathcal{B}_{4R}$; \textbf{3)} The feasible set $\mathcal{C}_t$ is non-empty and contained in $\mathcal{B}_R$ for all $t$.
\end{assumption}
We note that these assumptions are standard in online optimization \citep[][Ch.~3]{Hazan}. The learner's task is nontrivial even in the case where the feasible set is time invariant. 
If the feasible set is time varying, additional assumptions are required that restrict the amount that the feasible set is allowed to change. 
These two assumptions, see Part 2 i) and ii) below, are described by the following interaction protocol between the learner and the environment:

\begin{assumption}\label{as:ip}(Interaction protocol) 
At each time step $t\in\{1,\dots,T\}$:\\
\textbf{1)} the learner chooses $x_t$;\\
\textbf{2)} the environment chooses $f_t\in \mathcal{F}$ and $g_t\in\mathcal{G}$ such that i) $||g_{t}(x)-g_{t-1}(x)||_\infty =\mathcal{O}(1/t)$, uniformly for all $x\in\mathcal{B}_{4R}$, with $||\cdot||_\infty$ the $\ell_{\infty}$ norm, and ii) $\mathcal{C}_t$ is contained in $\mathcal{Q}_t:=\cap_{\ell=0}^{t-1}\mathcal{S}_{\ell}$, where $\mathcal{S}_t:=\{x\in \mathbb{R}^n~|~G(x_t)^{\T}(x-x_t)\geq0\}$ is a cone centered at $x_t$ for $t\geq1$ and $\mathcal{S}_{0}=\mathbb{R}^{n}$ (the situation is illustrated in Figure~\ref{fig:protocol}, more details are presented in Appendix~\ref{app:polyhedral_intersection});\\
\textbf{3)} the environment reveals to the learner partial information on cost $f_t(x_t)$, $\nabla f_t(x_t)$ and all violated constraints $g_{t,i}(x_t)$, $\nabla g_{t,i}(x_t)$ for $i\in I(x_t)$.
\end{assumption}

\begin{wrapfigure}[14]{r}{0.25\textwidth}
\vspace{-1em}
  \includegraphics[width=\linewidth]{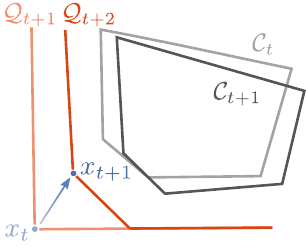}\vspace{-0.5em}
  \caption{\small{At each time step, the feasible set $\mathcal{C}_{t}$ contained in a polyhedral intersection $\mathcal{Q}_t$} changes slightly and is only partially revealed.}
  \label{fig:protocol}
\end{wrapfigure}

The requirements i) $||g_t-g_{t-1}||_\infty=\mathcal{O}(1/t)$; and ii) $\mathcal{C}_t \subset \mathcal{Q}_t$ restrict the feasible sets that the environment can choose.
We note that despite the fact that $||g_t-g_{t-1}||_\infty=\mathcal{O}(1/t)$, $||g_1-g_t||_\infty=\Theta(\ln(t))$, which means that the sequence of functions $g_t$ that defines $\mathcal{C}_t$ does \emph{not} converge in general. 
As a result, $C_t$ may evolve in such a way that the initial iterates $x_1, x_2, \dots, x_{t_0}$ achieve a large cost compared to $\min_{x\in \mathcal{C}_T} \sum_{t=1}^{T} f_t(x)$, as these are constrained by the sets $\mathcal{C}_1, \mathcal{C}_2, \dots, \mathcal{C}_{t_0}$, which may be far away from $\mathcal{C}_T$. 
The second requirement ii) $\mathcal{C}_t\subset \mathcal{Q}_t$ avoids this situation and is therefore key for obtaining an $\mathcal{O}(\sqrt{T})$ regret.

Our setup differs from traditional online convex optimization~\citep{Zinkevich03} in the following two important ways: 
\clearpage

i)~The environment chooses not only the functions $f_t$ but also the nonlinear constraint functions $g_t$, ii)~even if $g_t$ is time-invariant, i.e., $g_t=g$ for all $t$ the learner has only access to local information about the feasible set.

That is, the information about the feasible set is only revealed piece-by-piece and needs to be acquired by the agent through repeated queries of a constraint violation oracle.

We propose an online algorithm that despite the lack of information about the feasible set, achieves $\mathcal{O}(\sqrt{T})$ regret, and will derive explicit non-asymptotic bounds for the regret and the convergence to $\mathcal{C}_T$. We thus conclude that our algorithm matches the performance of traditional online projected gradients or Frank-Wolfe schemes, while requiring substantially less information about the feasible set and allowing it to be time-varying. 
Perhaps equally important is the fact that instead of performing projections onto the full feasible set at each iteration, our algorithm only optimizes over a local sparse linear approximation. 
If constraints are nonlinear, which includes norm-constraints or constraints on the eigenvalues of a matrix, optimizing over the full feasible set at each iteration can be computationally challenging.

\subsection{Related Work}

Online learning has its roots in online or recursive implementations of algorithms, where due to the piece-by-piece availability of data, algorithms are often analyzed in a non i.i.d.\ setting. 
A central algorithm is the multiplicative weights scheme \citep{Freud}, where a decider repeatedly chooses between a finite or countable number of options with the aim of minimizing regret. 
This online learning model not only offers a unifying framework for many classical algorithms \citep{Blum}, but represents a starting point for online convex optimization \citet{Hazan}, and adversarial bandits \citep{Tor}. Our approach extends this line of work by allowing the environment to not only choose the objective functions $f_t$, but also the constraints $g_{t}$.
Due to the fact that our learner only obtains local information about the feasible set, our work is somewhat related to \citet{Krause19,Zhou,Garber,Mhammedi}, where the aim is to reduce the computational effort of performing online projected gradient steps or Frank-Wolfe updates. 
More precisely, \citet{Krause19} propose an algorithm that directly approximates projections, while requiring multiple queries of the constraint functions and their gradients. 
A slightly different constraint violation oracle is assumed in \citet{Garber}, where the learner can query separating hyperplanes between a given infeasible point and the feasible set.
Algorithmically, both \citet{Garber} and \citet{Krause19} depart from online gradient descent, where the latter computes projections via an approximate Frank-Wolf-type scheme. 
An alternative is provided by \citet{Mhammedi} and \citet{Zhou}, where optimizations over the entire feasible set are simplified by querying only a set membership oracle based on the Minkowski functional. 
While our approach also avoids projections or optimizations over the entire feasible set, we introduce a different constraint violation oracle that returns a local sparse linear approximation of the feasible set. 
We call the constraint violation oracle only once every iteration and do not require a two-step procedure that involves multiple oracle calls. 
In addition, we also allow for adversarial time-varying constraints.

In addition, there has been important recent work that developed online optimization algorithms with constraints. In contrast to the primal formulation of our algorithm, these works are based on primal-dual formulations, where the algorithm is required to satisfy constraints on average, so called long-term constraints. 
The research can be divided into two lines of work \citet{W1,W2,W3} and \citet{S1,S2} that use a set of weaker and stricter definitions for constraint violations and investigate time-invariant constraints, which contrasts our formulation that includes time-varying constraints. 
A third line of work by \citet{MannorTY09,T1,T2,T3,T4,CaoL19,LiuWHF22} focuses on time-varying constraints, where, however, the following weaker notion of constraint violation is used: $\sum_{t=1}^{T} g_t(x_t) \geq -c \sqrt{T}$, where $t$ refers to time and $x_t$ to the learner's decision. This metric allows constraint violations for many iterations, as long as these are compensated by strictly feasible constraints (in the worst case even with a single feasible constraint with a large margin). 
In contrast, our algorithm satisfies $g_t(x_t)\geq -c/\sqrt{t}$ for all iterations $t\in \{1,\dots,T\}$, where $c$ is an explicit constant independent of the dimension of the decision variable and the number of constraints.
This means that we can explicitly bound the constraint violation at every iteration, whereas infeasible and strictly feasible iterates cannot compensate each other.

An important distinction to \citet{T2} is given by our performance metric (see also the discussion in \citet{T2} and \citet{LiuWHF22}).
On the one hand, the work by \citet{T1,CaoL19,LiuWHF22} use $\sum_{t=1}^{T} f_t(x_t)- \sum_{t=1}^{T} f_t(x_{t}^{\prime})$ as a performance measure, where the iterates $x_t$ are required to satisfy $\sum_{t=1}^{T} g_t(x_t) \geq -c \sqrt{T}$ and the optimal solutions $x_{t}^{\prime}$ satisfy $x_{t}^{\prime}\in\argmin_{x\in\mathcal{C}_t}f_t(x)$.
On the other hand, the work by \citet{T2,T3,T4} use $\sum_{t=1}^{T} f_t(x_t)- \sum_{t=1}^{T} f_t(x_{1:T}^{\star})$ as a performance measure, where the iterates $x_t$ are required to satisfy $\sum_{t=1}^{T} g_t(x_t) \geq -c \sqrt{T}$ and the optimal solution $x_{1:T}^{\star}$ satisfies $g_t(x_{1:T}^{\star})\geq 0$ \emph{for all} $t\in\{1,\dots,T\}$.
This leads to a major asymmetry in the way regret is measured: while the iterates of the online algorithm only need to satisfy a cumulative measure of constraint violation, the benchmark $x_{1:T}^{\star}$, which represents the best fixed decision in hindsight, is required to satisfy \emph{all} constraints $g_t(x_{1:T}^{\star})\geq 0$ for $t=\{1,\dots,T\}$. 
In contrast, the performance metric introduced in \eqref{eq:regret} is symmetric and imposes the same constraint $x\in\mathcal{C}_{T}$ (approximately) on the learner's decision $x_T$ and (exactly) on the benchmark $x_{T}^*$.
These features make our algorithm a valuable addition to the algorithmic toolkit of online constrained optimization.

\citet{CastiglioniCMR022} studied the following asymmetric setting with adversarial environment, benchmark $x_{T}^{\star}$ belonging to $\arg\min_{x\in\mathcal{X}}\sum_{t=1}^{T}f_{t}(x)$ subject to $\frac{1}{T}\sum_{t=1}^{T}g_{t}(x)\geq0$, online iterates $x_t$ satisfying $\frac{1}{T}\sum_{t=1}^{T}g_{t}(x_{t})\geq-\mathcal{O}(1/\sqrt{T})$, and regret $\sum_{t=1}^{T}f_{t}(x_t)-\sum_{t=1}^{T}f_{t}(x_{T}^{\star})\leq \mathcal{O}(\sqrt{T})$.
Their benchmark and regret formulation can be obtained as a special case of our formulation with time-averaged constraints, that is, when our $g_T(x)$ is chosen as $\frac{1}{T}\sum_{t=1}^T g_t(x)$.
In contrast, our iterate $x_T$ satisfies $\frac{1}{T}\sum_{t=1}^T g_t(x_T)\geq-\mathcal{O}(1/\sqrt{T})$, a constraint that is asymptotically the same as the one satisfied by the benchmark $x_{T}^{\star}$.
We further note that they introduced a parameter $\rho=\sup_{x\in\mathcal{X}}\min_{t\in[T]}\min_{i\in[m]}g_{t,i}(x)$, which is required to be positive and \emph{known} to the algorithm for achieving $\mathcal{O}(\sqrt{T})$ regret. 
Notably $\rho>0$ implies that the intersection of all feasible sets is non-empty, which is a strong assumption (as is knowledge about the parameter $\rho$).
In our formulation with \textit{time-averaged} constraints, Assumption~\ref{as:ip} reduces to the requirement that the feasible set $\mathcal{C}_t$ belongs to a polyhedral intersection $\mathcal{Q}_t$, which does not require a non-emtpy intersection of all $\mathcal{C}_t$ (has a geometrical interpretation and the assumption $||g_{t}-g_{t-1}||_\infty = \mathcal{O}(1/t)$ is automatically satisfied). 
Thus, there are situations, where the regret bound from \citet{CastiglioniCMR022} becomes vacuous (for $\rho=0$), while our method still provably achieves $\mathcal{O}(\sqrt{T})$ regret.
Additional differences are that \citet{CastiglioniCMR022} considers primal-dual methods and assumes that all constraints are revealed after every iteration, whereas our method is primal-only and has only partial information about all violated constraints.
The latter point reduces computation and simplifies projections onto the velocity polyhedron, but requires a nontrivial inductive argument for establishing $\mathcal{O}(\sqrt{T})$ regret.

Other relevant related studies have investigated online learning problems with supply/budget constraints.
In these settings, the decision maker must choose a sequence of actions that maximizes their expected reward while ensuring that a set of resource constraints are not violated.
The process terminates either after a pre-specified time horizon has been reached or when the total consumption of some resource exceeds its budget.
\citet{BadanidiyuruKS18} introduced the bandits with knapsacks framework, which considers bandit feedback, stochastic objective and constraint functions. 
They proposed an optimal algorithm for this problem, which was later improved by \citet{AgrawalD14,AgrawalD19} and \citet{ImmorlicaSSS22}.
\citet{ImmorlicaSSS22} introduced the adversarial bandits with knapsacks setting and showed that an appropriate benchmark for this setting is the best fixed distribution over arms. 
Since no-regret is no longer possible under this benchmark, they provide no-$\alpha$-regret guarantees for their algorithm.

An important special case of our online learning model arises when the environment is represented by an adversarial player that competes with the learner. This corresponds to a repeated generalized Nash game due to the constraint that couples the decisions of the learner and its adversary. 
If the adversary plays best response, the resulting equilibria are characterized by quasi-variational inequalities \citep{GeneralizedNash} and there has been important recent work, for example by \citet{FirstOrderAlgGNE, Kim, FacchineiPenalty} that proposes different gradient and penalty methods for solving these inequalities.
Our approach adopts a different perspective, rooted in online learning, which allows us to derive non-asymptotic convergence results for a first-order gradient-based algorithm that can be implemented in a straightforward manner.
Our approach is also inspired by the recent work of \citet{MJ21}, who propose a similar algorithm for the offline setting.

\subsection{Main Contributions}\label{subsec:MainContributions}

We give an online optimization scheme under \textit{unknown} non-linear constraints that achieves an optimal $\mathcal{O}(\sqrt{T})$ regret and converges to the latest feasible set at a rate of $\mathcal{O}(1/\sqrt{T})$. 
There are two variants of our problem formulation: The first deals with situations where constraints are unknown but fixed, the second allows constraints to be chosen in a time-varying and adversarial manner.

Our algorithm, named Constraint Violation Velocity Projection (CVV-Pro), has the following features:

1. It assumes access to a new type of oracle, which on input $x_t$, returns partial information on all currently violated constraints.
Namely, the value $g_{t,i}(x_t)$ and the gradient $\nabla g_{t,i}(x_t)$ for \textit{all} $i\in I(x_t)$.

2. It projects an adversarially generated negative cost gradient $-\nabla f_t(x_t)$ onto a velocity polyhedron $V_{\alpha}(x_t):=\big\{ v\in\mathbb{R}^{n} ~|~[\nabla g_{i}(x_{t})]^{\T}v \geq-\alpha g_{i}(x_{t}),\ \forall i\in I(x_{t})\big\}$.
Due to the linear and local structure of $V_{\alpha}(x_t)$, the projection can be computed efficiently.

3. In contrast to standard online methods that project in each round a candidate decision onto the feasible set, our method trades off feasibility for efficiency.
In particular, it produces a sequence of decisions that converges at a rate of $\mathcal{O}(1/\sqrt{T})$ to the latest feasible set.

4. Our method handles time-varying adversarial constraints $g_t$, provided a decreasing rate of change $||g_{t+1}-g_t||_\infty\leq\mathcal{O}(1/t)$ and that each feasible set $\mathcal{C}_t$ belongs to $\mathcal{Q}_t$ (see Assumption~\ref{as:ip}).
As we show in Section~\ref{subsec:Problem-Formulation}, an important special case where the assumption of decreasing rate of change is satisfied is given by $g_t=\frac{1}{t}\sum_{j=1}^{t}\tilde{g}_j$, i.e., when $g_t$ represents an average of constraints $\tilde{g}_t$ over time.

\subsection{Outline}
Section~\ref{sec:OnlineLearning} describes our algorithm and considers the situation where $g_t$ is time invariant. This sets the stage for our main results in Section~\ref{sec:Time-Varing-Constraints} that provide regret guarantees for our new online convex optimization setting with non-stationary, nonlinear, and unknown constraints. An important and interesting application of our algorithm are generalized Nash equilibrium problems, as will be illustrated with a numerical experiment in Section~\ref{sec:Simulation example}. The experiment will also highlight that the numerical results agree with the theoretical predictions.

\section{Online learning under unknown, time-invariant, and nonlinear constraints}\label{sec:OnlineLearning}

\subsection{Online Gradient Descent}

Online gradient descent~\citep[][Ch.~3.1]{Hazan} is a classical and perhaps the simplest algorithm that achieves optimal $\mathcal{O}(\sqrt{T})$ regret for the setting of a compact, convex, time-invariant, and a priori known feasible set.
It consists of the following two operations:
i)~$y_{t+1}=x_{t}-\eta_{t}\nabla f_{t}(x_t)$ takes a step from the previous point in the direction of the previous cost gradient; and 
ii)~$x_{t+1}=\mathrm{Proj}_{\mathcal{C}}(y_{t+1})$ projects $y_{t+1}$ back to the feasible set $\mathcal{C}$, as $y_{t+1}$ may be infeasible.

In this section, we generalize the online gradient descent algorithm to the setting where the feasible set is unknown a priori and has to be learned through repeated queries of a constraint violation oracle that only reveals local information.

\subsection{Overview}

In Section~\ref{subsec:CVV-Pro}, we present the pseudo code of our algorithm.
In Section~\ref{subsec:structural_result}, we give a structural result showing that Algorithm~\ref{alg:cvv-pro} under Assumption~\ref{as:fg} and a bounded iterate assumption guarantees an optimal $\mathcal{O}(\sqrt{T})$ regret and converges to the feasible set at a rate of $\mathcal{O}(1/\sqrt{T})$.
In Appendix~\ref{sec:algorithmic_result}, we show that the bounded iterate assumption can be enforced algorithmically, by introducing an additional hypersphere constraint that attracts the sequence $\{x_{t}\}_{t\geq1}$ to a fixed compact set.

\subsection{Constraint Violation Velocity Projection (CVV-Pro)}\label{subsec:CVV-Pro}

We present below the pseudocode of Algorithm~\ref{alg:cvv-pro} 
for a fixed horizon length $T$,
as it is standard in the literature~\citep{Hazan}. However, we note that
our algorithm is oblivious to the horizon length $T$, i.e., it can run
for any number of iterations without knowing $T$ a priori.
\begin{figure*}
	\centering
	\includegraphics[width=0.7\linewidth]{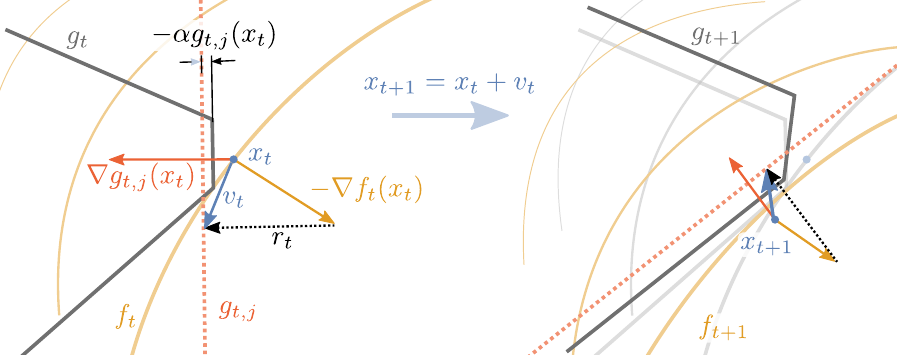}
	\caption{\small Illustration of the proposed (CVV-Pro) algorithm.
		\textbf{Left:} the constraint $g_{t,j}$ is violated by the {\color{ourblue} current solution $x_t$}.
		The {\color{ourorange}cost gradient $-\nabla f_t(x_t)$} is projected onto the {\color{ourred} hyperplane (moved by $-\alpha g_{t,j}(x_t)$) with normal vector $\nabla g_{t,j}(x_t)$}.
        This yields $r_t$ (see Section~\ref{subsec:thm:main}), and results in the {\color{ourblue}velocity projection $v_t$} ($\eta=1$ for clarity). 
		\textbf{Right:} next iteration with updated {\color{ourblue}$x$}, where both {\color{ourorange}$f$} and {\color{black!70!white}$g$} are changed. 
        Then the procedure is applied recursively.
	}
\end{figure*}

{\small
\begin{algorithm}[H]
	\caption{Constraint Violation Velocity Projection (CVV-Pro)}
	\label{alg:cvv-pro}
	\begin{algorithmic}[1]
            \State {\bfseries Requirements:} See Assumption~\ref{as:fg}
            
            \State {\bfseries Input:} $\alpha>0$
		
		\State {\bfseries Initialization:} Step sizes $\big\{\eta_{t}=\frac{1}{\alpha\sqrt{t}}\big\}_{t\geq1}$
		
		\For {$t=1$ {\bfseries to} $T$}
		
		\State \textbf{Play} $x_{t}$ \textbf{and observe}:
		
		\State cost information $f_{t}(x_{t}), \nabla f_{t}(x_{t})$ and 
		  constraint information $\big\{\big(g_{i}(x_{t}),\nabla g_{i}(x_{t})\big)\}_{i\in I(x_{t})}$
		
		\State \textbf{Construct the velocity polyhedron as follows:}
		\[
		V_{\alpha}(x_{t}):=\big\{ v\in\mathbb{R}^{n} ~|~[\nabla g_{i}(x_{t})]^{\T}v \geq-\alpha g_{i}(x_{t}),\ \forall i\in I(x_{t})\big\},
		\]
		
		\State \textbf{Solve the velocity projection problem:} $v_{t}=\argmin_{v\in V_{\alpha}(x_{t})}\frac{1}{2}\lVert v+\nabla f_{t}(x_{t})\rVert^{2}$ \label{eq:velocity_projection}
		
		\State \textbf{Update:} $x_{t+1}=x_{t}+\eta_{t}v_{t}$
            \EndFor
	\end{algorithmic}
\end{algorithm}}

Let $x\in\mathcal{C}$ be an arbitrary decision. 
We show in Claim~\ref{clm:concave} that $\alpha(x-x_t)\in V_{\alpha}(x_t)$.
Hence, the velocity polyhedron $V_{\alpha}(x_t)$ is always non-empty and well defined.

\subsection{Structural Result}\label{subsec:structural_result}

Here, we show that Algorithm~\ref{alg:cvv-pro} under Assumption~\ref{as:fg} and a bounded iterate assumption, guarantees an optimal $\mathcal{O}(\sqrt{T})$ regret and converges to the feasible set at a rate $\mathcal{O}(1/\sqrt{T})$.
The bounded iterate assumption will be removed subsequently, which however, will require a more complex analysis.

\begin{theorem}[Structural]\label{thm:main}
Suppose Assumption~\ref{as:fg} holds and in addition $x_{t}\in\mathcal{B}_R$ for all $t\in\{1,\dots,T\}$.
Then, on input $\alpha = L_{\mathcal{F}}/R$, Algorithm \ref{alg:cvv-pro} with step sizes
$\eta_{t}=\frac{1}{\alpha\sqrt{t}}$ guarantees the following for all $T\geq1$: 
 \vspace{-5pt}
 \begin{itemize}[noitemsep, topsep=0pt, leftmargin=1pc]
 \item[]\textbf{(regret)}\quad\quad $\sum_{t=1}^{T}f_{t}(x_{t})-\min_{x\in\mathcal{C}}\sum_{t=1}^{T}f_{t}(x)\leq 18L_{\mathcal{F}}R\sqrt{T}$;
 \item[]\textbf{(feasibility)}\, $g_{i}(x_{t})\geq-8\left[\frac{L_{\mathcal{G}}}{R}+2\beta_{\mathcal{G}}\right]\frac{R^{2}}{\sqrt{t}}$,\quad for all $t\in\{1,\dots,T\}$ and $i\in\{1,\dots,m\}$.
 \end{itemize}
\end{theorem}

\subsection{Proof Sketch of Theorem \ref{thm:main}}\label{subsec:thm:main}

Our analysis establishes, in two steps, an important geometric property that connects the convex costs and the concave constraints via the velocity polyhedron $V_{\alpha}(x_t)$.
More precisely, we show that the inner product $-r_t^{\T}(x_{T}^{\star}-x_t)\leq0$ for all $t\in\{1,\dots,T\}$.
This property will be crucial for deriving the regret and feasibility bounds.

In the first step, we leverage the constraints' concavity and show that the vector $\alpha(x_{T}^{\star}-x_{t})$ belongs to the velocity polyhedron $V_{\alpha}(x_{t})$.
\begin{claim}\label{clm:concave} 
Suppose $g_{i}$ is concave for every $i\in\{1,\dots,m\}$. Then $\alpha(x-x_{t})\in V_{\alpha}(x_{t})$ for all $x\in\mathcal{C}$. In addition, $x_{t}\not\in\mathrm{int}(\mathcal{C})$ implies $[\nabla g_{i}(x_{t})]^{\T}[x-x_{t}]\geq 0$ for all $x\in\mathcal{C}$.
\end{claim}
\begin{proof}
Let $x\in\mathcal{C}$ be an arbitrary feasible decision, satisfying
$g_{i}(x)\geq0$ for all $i\in\{1,\dots,m\}$. Since $g_{i}$ is concave,
we have $g_{i}(x_{t})+[\nabla g_{i}(x_{t})]^{\T}[x-x_{t}]\geq g_{i}(x)\geq0$
and thus $[\nabla g_{i}(x_{t})]^{\T}[x-x_{t}]\geq - g_{i}(x_{t})$.
The second conclusion follows by $x_{t}\not\in\mathrm{int}(\mathcal{C})$, which implies $g_i(x_{t})\leq0$.
\end{proof}
In the second step, we show that $r_{t}^{\T}(x_{t}-x^{\star})\leq0$, where $r_{t}=v_{t}+\nabla f_{t}(x_{t})$ is such that $-r_{t}$ belongs to the normal cone $N_{V_{\alpha}(x_{t})}(v_{t})$ of the velocity polyhedron $V_{\alpha}(x_{t})$ evaluated at the projection $v_{t}$.
\begin{lemma}[Main]\label{lem:main} 
	Let $v_{t}$ be the projection of $-\nabla f_{t}(x_{t})$
	onto the polyhedron $V_{\alpha}(x_{t})$ such that $v_{t}=r_{t}-\nabla f_{t}(x_{t})\in V_{\alpha}(x_{t})$,
	where $-r_{t}\in N_{V_{\alpha}(x_{t})}(v_{t})$. Then, $-r_{t}^{\T}(x-x_{t})\leq0$
	for all $x\in\mathcal{C}$.
\end{lemma}
\begin{proof}
	By definition, the normal cone $N_{V_{\alpha}(x_{t})}(v_{t})$ is given by $\{u\in \mathbb{R}^n~|~ u^{\T}(v-v_{t})\leq0,\ \forall v\in V_{\alpha}(x_{t})\}$.
	Then, by construction $-r_{t}\in N_{V_{\alpha}(x_{t})}(v_{t})$ and thus it holds for every $v\in V_{\alpha}(x_{t})$ that $-r_{t}^{\T}[v-v_{t}]\leq0$.
	The proof proceeds by case distinction:
	
	\textbf{Case 1.} Suppose $x_{t}$ is in the interior of $\mathcal{C}$.
	Then, $I(x_{t})=\emptyset$, which implies $-\nabla f_{t}(x_{t})\in V_{\alpha}(x_{t})=\mathbb{R}^{n}$ and thus $r_{t}=0$.
	
	\textbf{Case 2.} Suppose $x_{t}$ is on the boundary or outside of $\mathcal{C}$, i.e., $I(x_{t})\neq\emptyset$. 
	By Claim \ref{clm:concave}, we have $[\nabla g_{i}(x_{t})]^{\T}[x-x_{t}]\geq 0$ for all $x\in\mathcal{C}$.
	By construction, $v_{t}\in V_{\alpha}(x_t)$ and thus $v(x)=v_t+x-x_t\in V_{\alpha}(x_t)$.
	The statement follows by applying $v=v(x)$ to $-r_{t}^{\T}[v-v_{t}]\leq0$.
\end{proof}
\paragraph{Regret.} To establish the first conclusion of Theorem~\ref{thm:main} (regret), we combine the preceding geometric property with the analysis of online gradient descent. 
Since $f_t\in\mathcal{F}$ is convex, we upper bound the regret in terms of the gradient of $f_t$, namely $\sum_{t=1}^{T}f_{t}(x_{t})-f_{t}(x_{T}^{\star}) \leq \sum_{t=1}^{T} [\nabla f_{t}(x_{t})]^{\T}(x_{t}-x_{T}^{\star})$ and then we show that the following inequality holds
\begin{eqnarray}
	[\nabla f_{t}(x_{t})]^{\T}(x_{t}-x_{T}^{\star})-\frac{\eta_{t}}{2}\lVert v_{t}\rVert^{2} & = & 	r_{t}^{\T}(x_{t}-x_{T}^{\star})+\frac{\lVert x_{t}-x_{T}^{\star}\rVert^{2}-\lVert x_{t+1}-x_{T}^{\star}\rVert^{2}}{2\eta_{t}} \nonumber\\
	& \leq & \frac{\lVert x_{t}-x_{T}^{\star}\rVert^{2}-\lVert x_{t+1}-x_{T}^{\star}\rVert^{2}}{2\eta_{t}}.\label{eq:OGD_regret}
\end{eqnarray}
Moreover, in Appendix~\ref{app:thm_main} (see Lemma \ref{lem:velocity}), we upper bound the velocity $\lVert v_{t}\rVert\leq\alpha\lVert x_{T}^{\star}-x_{t}\rVert+2\lVert\nabla f_{t}(x_{t})\rVert$.
Combining Assumption~\ref{as:fg} and $x_{t}\in\mathcal{B}_R$ yields a uniform bound $\lVert v_{t}\rVert\leq\mathcal{V}_{\alpha}$, where for $\alpha=L_{\mathcal{F}}/R$ we set $\mathcal{V}_{\alpha}:=4L_{\mathcal{F}}$.
The desired regret follows by a telescoping argument and by convexity of the cost functions $f_t\in\mathcal{F}$.	

\paragraph{Feasibility.} For the second conclusion of Theorem~\ref{thm:main} (convergence to the feasible set at a rate of $1/\sqrt{T}$), we develop a non-trivial inductive argument that proceeds in two steps. 
In Appendix~\ref{app:thm_main} (see Claim~\ref{clm:const_violation}), we give a structural result that bounds the constraint functions from below. 
In particular, for every $i\in I(x_{t})$ we have $g_{i}(x_{t+1})\geq(1-\alpha\eta_{t})g_{i}(x_{t})-\eta_{t}^{2}\mathcal{V}_{\alpha}^{2}\beta_{\mathcal{G}}$ and for every $i\not\in I(x_{t})$ it holds that $g_{i}(x_{t+1})\geq-\eta_{t+1}\mathcal{V}_{\alpha}[2L_{\mathcal{G}}+\mathcal{V}_{\alpha}\beta_{\mathcal{G}}/\alpha]$.

Using an inductive argument, we establish in Appendix~\ref{app:thm_main} (see Lemma~\ref{lem:feasibility_convergence}) the following lower bound: $g_{i}(x_{t})\geq-c\eta_{t}$ where $c=2\mathcal{V}_{\alpha}(L_{\mathcal{G}}+\mathcal{V}_{\alpha}\beta_{\mathcal{G}}/\alpha)$.
Choosing $\alpha=L_{\mathcal{F}}/R$ implies that $\mathcal{V}_{\alpha}=4L_{\mathcal{F}}$. 
Then, the desired convergence rate to the feasible set follows for the step size $\eta_{t}=\frac{1}{\alpha\sqrt{t}}$, since 
\[
-c\eta_{t}=-\frac{2\mathcal{V}_{\alpha}}{\alpha\sqrt{t}}\Big[L_{\mathcal{G}}+\frac{\beta_{\mathcal{G}}\mathcal{V}_{\alpha}}{\alpha}\Big]=-8\Big[\frac{L_{\mathcal{G}}}{R}+4\beta_{\mathcal{G}}\Big]\frac{R^{2}}{\sqrt{t}}.
\]

\section{Online Learning under Adversarial Nonlinear Constraints \label{sec:Time-Varing-Constraints}}

\subsection{Problem Formulation}\label{subsec:Problem-Formulation}

In this section, we consider an online optimization problem with adversarially generated time-varying constraints.
More precisely, at each time step $t$, the learner receives partial information on the current cost $f_t$ and feasible set $\mathcal{C}_t$, and seeks to minimize \eqref{eq:regret}.
To make this problem well posed, we restrict the environment such that each feasible set $\mathcal{C}_t$ is contained in $\mathcal{Q}_t$ (see Section~\ref{sec:Introduction}) and the rate of change between consecutive time-varying constraints \textit{decreases} over time.
We quantify a sufficient rate of decay with the following assumption.

\begin{assumption}[TVC Decay Rate]\label{as:TVC_decay}
We assume that the adversarially generated sequence $\{g_{t}\}_{t\geq1}$ of time-varying constraints is such that for every $x\in\mathcal{B}_{4R}$ and all $t\geq1$, the following holds
$\Vert g_{t+1}(x)-g_{t}(x)\Vert_{\infty}\leq\frac{98}{t+16}\big[\frac{L_{\mathcal{G}}}{R}+3\beta_{\mathcal{G}}\big]R^{2}$.
\end{assumption}

We note that Assumption~\ref{as:TVC_decay} essentially only requires $\Vert g_{t+1}(x)-g_{t}(x)\Vert_{\infty}\leq\mathcal{O}(1/t)$, as $R$ can be chosen large enough such that the bound is satisfied.
Of course, $R$ will appear in our regret and feasibility bounds, but it will not affect the dependence on $t$ or $T$ (up to constant factors).

An important special case where Assumption~\ref{as:TVC_decay} is satisfied, is summarized in the following Lemma. The proof is included in Appendix~\ref{app:thm:avrTVC} (see Lemma~\ref{lem:linearity} and Lemma~\ref{lem:hgit}).
\begin{lemma}\label{lem:construct_average}
Suppose the functions $\tilde{g}_{t,i}$ satisfy Assumption~\ref{as:fg} and in addition there is a decision $x_{t,i}\in\mathcal{B}_R$ such that $\tilde{g}_{t,i}(x_{t,i})=0$ for every $t\geq1$ and $i\in\{1,\dots,m\}$.
Then the time-averaged constraints $g_{t,i}(x):=\frac{1}{t} \sum_{\ell=1}^t \tilde{g}_{\ell,i}(x)$ satisfy Assumption~\ref{as:fg} and Assumption~\ref{as:TVC_decay}.
\end{lemma}

\subsection{Velocity Projection with Attractive Hypersphere Constraint}

We show in Appendix~\ref{sec:algorithmic_result} that the second assumption in Theorem~\ref{thm:main}, namely, ``$x_{t}\in\mathcal{B}_R$ for all $t\geq1$'' can be enforced algorithmically.
We achieve this in two steps.

1) Algorithmically, we introduce an additional hypersphere constraint $g_{m+1}(x_{t})=\frac{1}{2}[R^{2}-\Vert x_{t}\Vert^{2}]$ that attracts the decision sequence $\{x_{t}\}_{t\geq1}$ to a hypersphere  $\mathcal{B}_{R}$ and guarantees that it always stays inside a hypersphere $\mathcal{B}_{4R}$ with a slightly larger radius.

More precisely, we augment the velocity polyhedron in Step 3 of Algorithm~\ref{alg:cvv-pro} as follows:
$V_{\alpha}^{\prime}(x_{t})=V_{\alpha}(x_{t})$  if $\Vert x\Vert\leq R$, otherwise
\[
V_{\alpha}^{\prime}(x_{t})=\{v\in V_{\alpha}(x_{t}) ~|~ [\nabla g_{m+1}(x_{t})]^{\T}v\geq-\alpha g_{m+1}(x_{t})\}.
\]
2) Analytically, we give a refined inductive argument in Appendix~\ref{app:thm:avrTVC} (see Lemma~\ref{lem:gmp_all}), showing that $g_{m+1}(x_{t})\geq-27R^{2}/\sqrt{t+15}$, $\Vert x_{t}\Vert\leq4R$ and $\Vert v_{t}\Vert\leq7L_{\mathcal{F}}$, for all $t\geq1$.

\subsection{Main Contribution}
Our main contribution is to show that Algorithm~\ref{alg:cvv-pro} with the augmented velocity polyhedron $V_{\alpha}^{\prime}(x_{t})$, achieves optimal $\mathcal{O}(\sqrt{T})$ regret and satisfies $g_T(x_T)\geq -\mathcal{O}(1/\sqrt{T})$ convergence feasibility rate. 
Due to space limitations, we defer the proof to Appendix~\ref{app:thm:avrTVC}.

\begin{theorem}[Time-Varying Constraints] \label{thm:avrTVC}
Suppose the functions $\{f_{t},g_{t}\}_{t\geq1}$ satisfy Assumptions~\ref{as:fg},~\ref{as:ip} and \ref{as:TVC_decay}.
Then, on input $R,L_{\mathcal{F}}>0$ and $x_{1}\in\mathcal{B}_{R}$, Algorithm~\ref{alg:cvv-pro} applied with $\alpha=L_{\mathcal{F}}/R$, augmented velocity polyhedron $V_{\alpha}^{\prime}(\cdot)$ and step sizes $\eta_{t}=\frac{1}{\alpha\sqrt{t+15}}$ guarantees the following for all $T\geq1$:
\vspace{-5pt}
\begin{itemize}[noitemsep, topsep=0pt, leftmargin=1pc]
\item[]\textbf{(regret)} \quad\quad $\sum_{t=1}^{T}f_{t}(x_{t})-\min_{x\in\mathcal{C}_{T}}\sum_{t=1}^{T}f_{t}(x)\leq246L_{\mathcal{F}}R\sqrt{T}$;	
\item[]\textbf{(feasibility)} \,\, $g_{t,i}(x_{t})\geq-265\left[\frac{L_{\mathcal{G}}}{R}+4\beta_{\mathcal{G}}\right]\frac{R^{2}}{\sqrt{t+15}}$,\quad  for all $t\in\{1,\dots,T\}$ and $i\in\{1,\dots,m\}$;
\item[]\textbf{(attraction)} \,\,$g_{m+1}(x_{t})\geq-27\frac{R^{2}}{\sqrt{t+15}}$,\quad for all $t\in\{1,\dots,T\}$.
\end{itemize}
\end{theorem}

Our regret analysis in Theorem~\ref{thm:avrTVC} builds upon the following key structural result that generalizes Lemma~\ref{lem:main} to time-varying constraints.
In particular, in Appendix~\ref{app:thm:avrTVC} (see Lemma~\ref{lem:time-varying-regret}), we show that given the feasible set $\mathcal{C}_T \subset \mathcal{Q}_T$, it holds for every $x\in\mathcal{C}_T$ that $-r_{t}^{\T}(x - x_{t}) \leq 0$ for all $t\in\{1,\dots,T\}$.
As a result, a similar argument as in \eqref{eq:OGD_regret} shows that the regret is bounded by $\mathcal{O}(\sqrt{T})$.

Moreover, we note that the linear and quadratic dependence on $R$ in Theorem~\ref{thm:avrTVC} is consistent in length units.
Let the radius $R$ be of length units $\ell$, then the Lipschitz constant $L_{\mathcal{F}}$, which can be viewed as the supremum over the $\ell_2$ norm of the gradient is of $1/\ell$ units, and the $\beta_{\mathcal{G}}$ smoothness constant (associated with Hessian) is of $1/\ell^2$ units. This means that the regret bound in Theorem~\ref{thm:avrTVC} has the same units as $f_t$, while the feasibility bound has the same units as $g_t$.

\section{Simulation examples}\label{sec:Simulation example}

Two-player games with shared resources are an excellent example for demonstrating the effectiveness and importance of our online learning framework.
We apply our algorithm and show numerical experiments that support our theoretical findings.

We choose random instances of a two player game with linear utility and constraints.
In particular, we consider the following optimization problem
\begin{align}\label{exp:game}
\min_{x\in\triangle_{n}} \max_{y\in\triangle_{n}} x^{\T}A y \quad \text{subject to}\quad C_{x}x+C_{y}y\leq 1,
\end{align}
where $\triangle_{n}=\{x\in\mathbb{R}^n ~|~\sum_{i=1}^{n}x_i=1, x\geq0\}$ is the probability simplex. Each component of the utility matrix $A\in\mathbb{R}^{n\times n}$ is sampled from the normal distribution and the constraint matrices $C_x,C_y\in[0,1]^{m\times n}$ have each of their components sampled uniformly at random from $[0,1]$.

\begin{figure}[h]
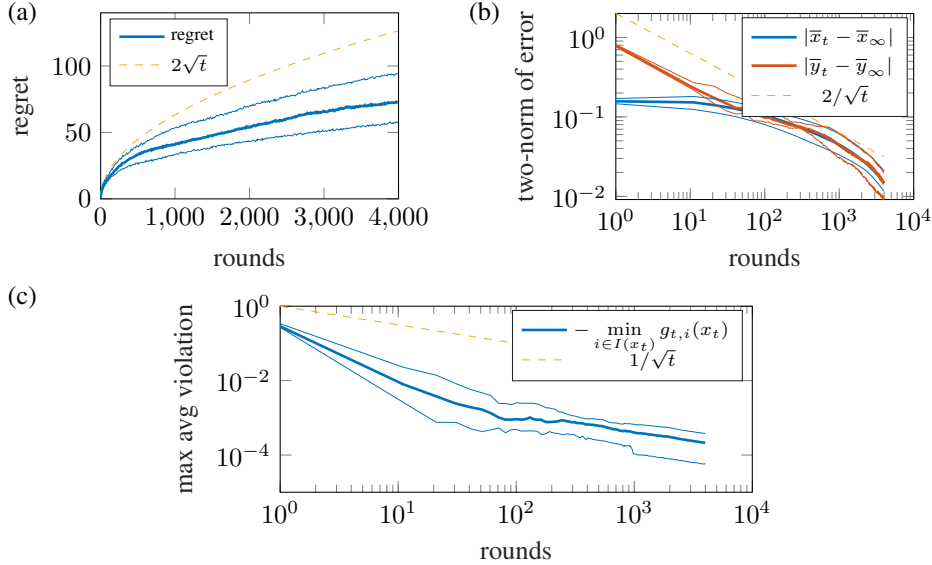

    \centering
     \begin{tabular}{@{}r@{\hspace{-1.5em}}cr@{}c@{}}
        (a)& & (b) \\[-1em]
        &\setlength\figurewidth{3.3cm}
        \input{plots_fix_alpha/regret.tex} &&
        \setlength\figurewidth{3.3cm}
        \input{plots_fix_alpha/convergence_xkyk_avg.tex} \\
        (c)\\[-1em]
        & \multicolumn{3}{c}{
        \setlength\figurewidth{3.3cm}
        \input{plots_fix_alpha/maxviolation.tex}}
     \end{tabular}
    \vspace{-1.1em}
    \caption{\small \textbf{(a)} The CVV-Pro algorithm is executed on five random instances of the two-player game with shared resources (Section~\ref{subsec:two-player-game}). The regret follows the predicted $\mathcal{O}(\sqrt{T})$ slope. 
    The thick line is the mean and the thin lines indicate the minimum and the maximum over the five runs.
    \textbf{(b)} The CVV-Pro algorithm achieves a convergence rate of $\mathcal{O}(1/\sqrt{t})$ for the averaged decisions $\overline{x}_t:=\frac{1}{t}\sum_{\ell=1}^{t}x_{\ell}$ towards $\overline{x}_{\infty}$.
    In our experiment, we set $\overline{x}_{\infty}:=\overline{x}_{10000}$. 
    Similar behavior is reported for the averaged decisions $\overline{y}_t$ of the adversary.
    \textbf{(c)} The maximal constraint violation expressed by $-\min_{i \in I(x_{t})}g_{t,i}(x_{t})$ converges at a rate of $\mathcal{O}(1/\sqrt{t})$, as predicted by our theoretical results.}
    \label{fig:CVVPro}
\end{figure}

\subsection{Online Formulation}\label{subsec:two-player-game}

The problem in \eqref{exp:game} can be modeled with our online learning framework \eqref{eq:regret} by choosing costs $f_{t}(x):=x^{\T} Ay_t$ and time-averaged resource constraints $g_T(x):=\frac{1}{T}\sum_{t=1}^{T}\tg{t}(x)$, where the function $\tg{t}(x):=1-C_{x}x-C_{y}y_{t}$.
Thus, the constraint in \eqref{exp:game} is included as an average over the past iterations of $y_t$.
The strategy for choosing $y_t$ will be described below and, as we will see, the average of $y_t$ over the past iterations converges.
This ensures that the feasible set $\mathcal{C}_t$ (defined in \eqref{eq:regret}) is slowly time-varying, while the averages of $x_t$ and $y_t$ over past iterates converge to equilibria in \eqref{exp:game}.
Further, by a refined version of Lemma~\ref{lem:construct_average} (see Lemma~\ref{lem:averageTVC} in Appendix~\ref{app:thm:avrTVC}), the time-averaged constraints $g_{T}(x)$ satisfy Assumption~\ref{as:TVC_decay}.

In each iteration, Algorithm~\ref{alg:cvv-pro} seeks to minimize the online problem and commits to a decision $x_{t}$.
The adversary computes the best response $\hat{y}_{t}$ with respect to the decision $x_{t}$ by solving $\argmax_{y\in \triangle_{n}} x_t^{\T} Ay$.
To make the dynamics more interesting, the adversary then commits with probability 0.8 to $\hat{y}_{t}$ and with probability 0.2 to a random decision $\xi_{t}$, i.e.,
$y_t=0.8\hat{y}_{t} + 0.2\xi_{t}$ where the random variable $\xi_{t}$ is sampled uniformly at random from the probability simplex $\triangle_{n}$.

As both players optimize over the probability simplex ($x,y\in\triangle_{n}$),
the sequence of decisions $\{x_t\}_{t\geq1}$ is automatically bounded.
Thus, we can apply Theorem~\ref{thm:avrTVC} with the original velocity polyhedron, as discussed in Appendix~\ref{sec:algorithmic_result}.
We implemented our algorithm with $\eta_t=1/(\alpha \sqrt{t})$ and $\alpha=100$.

\subsection{Experimental Results}

We report results from numerical simulations with decision dimension $n=100$, $m=10$ shared resource constraints, $T=4000$ iterations, and five independently sampled instances of the two-player game.
The learner's regret, depicted in Figure~\ref{fig:CVVPro}a, shows a clear correspondence with the theoretical prediction of $\mathcal{O}(\sqrt{T})$.
Figure~\ref{fig:CVVPro}b presents the maximal constraint violation $-\min_{i\in I(x_T)}\frac{1}{T}\sum_{t=1}^{T}\widetilde{g}_{t,i}(x_{T})$, which follows the predicted $\mathcal{O}(1/\sqrt{T})$ convergence rate.
We also conclude from Figure~\ref{fig:CVVPro}c that the learner's averaged decisions $\overline{x}_T=\frac{1}{T}\sum_{t=1}^{T}x_{t}$ converge at a rate of $\mathcal{O}(1/\sqrt{T})$. 
Similarly, the averaged decisions $\overline{y}_T$ of the adversary also converge at a rate of $\mathcal{O}(1/\sqrt{T})$.
We note that there is little variability in the results despite the different realizations of the matrices $A, C_x, C_y$.

\paragraph{Contrasting CVV-Pro and Online Gradient Descent}
In Appendix~\ref{app:OGD_CVVPro}, we show that our (CVV-Pro) algorithm outperforms the standard Online Gradient Descent algorithm in the two-player game from above.
In particular, our algorithm achieves a lower regret and a runtime improvement of about $60\%$.
Further, the percentage of violated constraints decreases rapidly and plateaus at $20\%$.

The amount of improvement in execution time is likely to be greater for higher-dimensional problems, where fewer constraints tend to be active at each iteration.
Moreover, when the constraints are nonlinear, which includes $\ell_p$ norm or spectral constraints, optimizing over the full feasible set can be computationally challenging.
In contrast, the velocity projection step in CCV-Pro is always a convex quadratic program with linear constraints, regardless of the underlying feasible set.

\section{Broader Impact}\label{sec:BroaderImpact}
It is important to emphasize that our work is theoretical, and the main contribution is to design and analyze a novel algorithm that combines techniques from the seemingly distant fields of online convex optimization (online gradient descent) and non-smooth mechanics (velocity space, see~\citet{MJ21}). Nevertheless, the list of potential applications includes, but is not limited to: adversarial contextual bandits with sequential risk constraints \citet{T4}, network resource allocation \citet{T1}, logistic regression \citet{CaoL19,LiuWHF22}, ridge regression and job scheduling \citet{LiuWHF22}, two-player games with resource constraints (Section~\ref{sec:Simulation example}), system identification and optimal control (Appendix~\ref{app:sec:more-apps}).

\section{Conclusion}\label{sec:Conclusion}

We propose an online algorithm that, despite the lack of information about the feasible set, achieves $\mathcal{O}(\sqrt{T})$ regret.
We further ensure convergence of violated constraint $-\min\{g_{T}(x_{T}),0\}$ at a rate of $\mathcal{O}(1/\sqrt{T})$ and derive explicit constants for all our bounds that hold for all $T\geq1$.
We thus conclude that our algorithm matches the performance of traditional online projected gradients or Frank-Wolfe schemes, while requiring substantially less information about the feasible set and allowing the feasible set to be time-varying.
Perhaps equally important is the fact that instead of performing projections onto the full feasible set at each iteration, our algorithm only optimizes over a local sparse linear approximation. 
We show the applicability of our algorithm in numeric simulations of random two-player games with shared resources.

\section*{Acknowledgements}
We acknowledge the support from the German Federal Ministry of Education and Research (BMBF) through the Tübingen AI Center (FKZ: 01IS18039B).
Georg Martius is a member of the Machine Learning Cluster of Excellence, EXC number 2064/1 – Project number 390727645.
Pavel Kolev was supported by the Cyber Valley Research Fund and the Volkswagen Stiftung (No 98 571). 
Michael Muehlebach thanks the German Research Foundation and the Branco Weiss Fellowship, administered by ETH Zurich, for the support.
We thank anonymous reviewers for comments, which helped improve the presentation of the paper.

\bibliographystyle{plainnat}
\bibliography{references}

\newpage
\appendix

\renewcommand{\thetable}{S\arabic{table}}
\renewcommand{\thefigure}{S\arabic{figure}}
\renewcommand{\theequation}{S\arabic{equation}}
\setcounter{table}{0}
\setcounter{figure}{0}
\setcounter{equation}{0}

\begin{center}
\Large
Supplementary Material for Online Learning under Adversarial Nonlinear Constraints
\end{center}

\section{Polyhedral Intersection}\label{app:polyhedral_intersection}

\begin{figure}[h]
    \centering
    \includegraphics[width=0.9\linewidth]{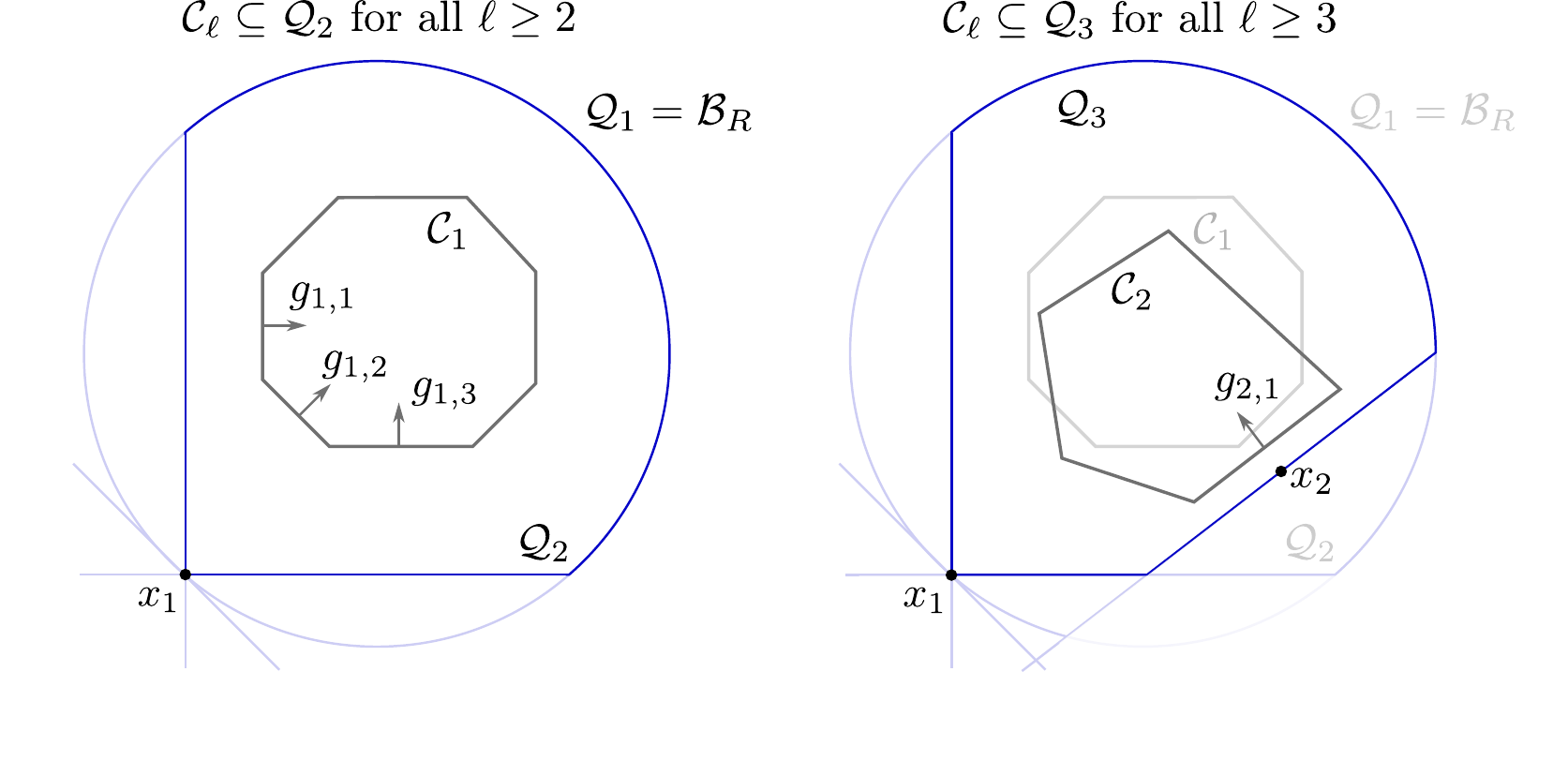}
    \caption{Illustration of polyhedral intersection.}
    \label{fig:PolyhedralIntersection}
\end{figure}

We present here a simplified geometric setting of the Assumption~\ref{as:ip} Part 2, where the time-varying constraints are linear $g_{t,i}(x)=g_{t,i}^{\T}x$.

In the interaction protocol of Assumption~\ref{as:ip}, the learner first commits a decision $x_1\in\mathcal{B}_R$.
Then the environment selects a feasible set $\mathcal{C}_1\subseteq\mathcal{Q}_1=\mathcal{B}_R$ (by Assumption 1.1 Part 3) and reveals to the learner a cost value $f_1(x_1)$, a gradient $\nabla f_1(x_1)$, and information for all violated constraints $(g_{1,i}(x_1),\nabla g_{1,i}(x_1))_{i=1}^{3}$.
This constraint violation information restricts the adversary to selecting successive feasible sets $\mathcal{C}_{\ell}$, for all $\ell\geq2$, from a polyhedral intersection $\mathcal{Q}_2=\mathcal{Q}_1\cap S_1$, where the cone intersection $S_1=\{z\in\mathbb{R}^n : G_1(x_1)^{\T}(z-x_1)\geq0\}$ and the gradient matrix $G_1(x_1)=[\nabla g_{1,i}(x_1)]_{i=1}^{3}$.

In the next iteration, the learner makes an update and commits a decision $x_2=x_1+\eta_1 v_1$.
Then, the process is repeated: the environment selects $\mathcal{C}_2\subseteq\mathcal{Q}_2$, reveals a cost value $f_2(x_2)$, a gradient $\nabla f_2(x_2)$ and constraint violation information $g_{2,1}(x_2),\nabla g_{2,1}(x_2)$.
All successive feasible sets $\mathcal{C}_{\ell}$, for all $\ell\geq3$, are restricted to belong to a polyhedral intersection $\mathcal{Q}_3=\mathcal{Q}_2\cap S_2$, where the cone intersection $S_2=\{z\in\mathbb{R}^n : G_2(x_2)^{\T}(z-x_2)\geq0\}$ and the gradient matrix $G_2(x_2)=[\nabla g_{2,1}(x_1)]$.

\section{Further Applications}\label{app:sec:more-apps}

We consider here a system identification and optimal control application where an agent must predict a sequence of actions to minimize costs and satisfy constraints.
Many real-world systems are subject to wear, tear and drift (e.g., sensors), which naturally leads to non-stationary costs and constraints, corresponding to slowly time-varying functions $f_t$ and $\{g_{t,i}\}_{i=1}^{m}$, respectively.
Furthermore, it is common in optimal control to know analytically both the dynamics model and the cost and constraint functions, so the gradients are naturally available.
Assuming access to a constraint violation oracle, the above scenario can be cast into our online problem formulation.
More specifically, in each episode $t$, an agent $\phi$ parameterized by weights $\theta_t\in\mathbb{R}^n$ generates a sequence of actions $\{x_{\ell}\}_{\ell=1}^{H}$ and upon their deployment in the environment, receives a cost value $f_t(\theta_t)$, gradient $\nabla f_t(\theta_t)$ and information for all violated constraints $\{(g_{t,i}(\theta_t), \nabla g_{t,i}(\theta_t)\}_{i\in I(\theta_t)}$.

\newpage
\section{Contrasting CVV-Pro and OGD: A Comparative Study}\label{app:OGD_CVVPro}

In this section, we compare the runtime performance and regret guarantee of the standard Online Gradient Descent (OGD) algorithm and our (CVV-Pro) algorithm in the two-player game setting (defined in Section~\ref{sec:Simulation example}).
More concretely, we consider shared constraints of the form $C_xx+C_yy\leq b$.
We report results from numerical simulations with decision dimension $n=1000$, $m=100$ shared resource constraints, capacity $b=1.3$, $T=2000$ iterations, and 5 independently sampled instances of the two-player game.
We report below the results:

\paragraph{Regret:} The 25th percentile of OGD has a higher regret around iteration 1400 than the function $5\sqrt{t}$ and stays above it.
In contrast, CVV-Pro achieves better regret, with the 75th percentile being strictly bounded by the function $5\sqrt{t}$, see Figure~\ref{fig:OGD_vs_CVVPro}a.

\paragraph{$\%$ Constraints Violation} In each iteration, CVV-Pro requires an oracle access only to the currently violated constraints.
The percentage of violated constraints first increases from $0.01\%$ to $57\%$ in the first four iterations, and then decreases rapidly to plateau at $20\%$, see Figure~\ref{fig:OGD_vs_CVVPro}b.

\paragraph{Runtime:}
In Figure~\ref{fig:OGD_vs_CVVPro}c, we report the average runtime per iteration for computing a projection.  
Since CVV-Pro solves the velocity projection problem with a decreasing number of constraints, it achieves a faster average runtime of $0.11\pm0.01$s compared to OGD, which requires solving the full projection problem each time and runs in $0.18\pm0.01$s.
Thus, for the two-player game with shared constraints, our algorithm CVV-Pro achieves a runtime improvement of around $60\%$ over OGD.
Further, we report in Figure~\ref{fig:OGD_vs_CVVPro}d the total cumulative runtime of CVV-Pro and OGD for computing the projection.

The amount of improvement in execution time is likely to be greater for higher-dimensional 
problems, where fewer constraints tend to be active at each iteration. 
Moreover, there are important situations, for example if constraints are non-convex, where projections are very difficult to compute (and/or might not even be well defined). 
In contrast, the velocity projection step in CCV-Pro is always a convex problem, regardless of whether the underlying feasible set is convex or not.

\begin{figure}[h]
    \centering
	\begin{tabular}{@{}r@{\hspace{-1.5em}}cr@{}c@{}}
        (a)& & (b) \\[-1em]
        & \setlength\figurewidth{3.75cm}
        \input{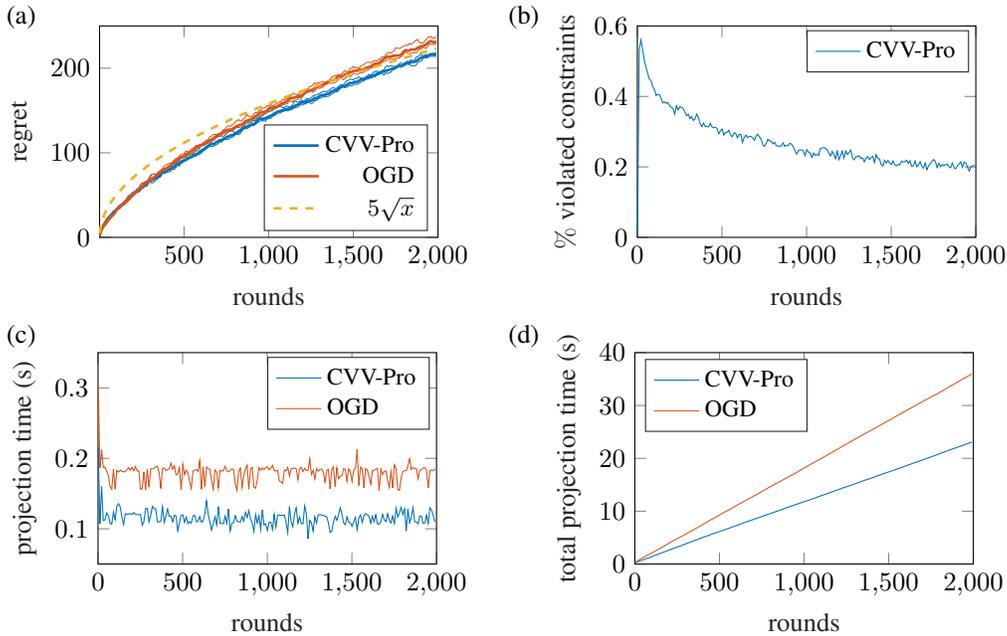} &&
        \setlength\figurewidth{3.75cm}
%
%
\definecolor{mycolor1}{rgb}{0.00000,0.44700,0.74100}%
\definecolor{mycolor2}{rgb}{0.85000,0.32500,0.09800}%
\definecolor{mycolor3}{rgb}{0.92900,0.69400,0.12500}%
\begin{tikzpicture}

\begin{axis}[%
width=1.2\figurewidth,
height=0.75\figurewidth,
at={(0\figurewidth,0\figurewidth)},
scale only axis,
xmin=0,
xmax=2000,
xlabel style={font=\color{white!15!black}},
xlabel={rounds},
ymin=0,
ymax=0.6,
ylabel style={font=\color{white!15!black}},
ylabel={$\%$ violated constraints},
axis background/.style={fill=white},
legend style={legend cell align=left, align=left, draw=white!15!black}
]
\addplot [color=mycolor1]
  table[row sep=crcr]{%
1	0.0109090909090909\\
11	0.530909090909091\\
21	0.56\\
31	0.533636363636364\\
41	0.502727272727273\\
51	0.483636363636364\\
61	0.465454545454545\\
71	0.449090909090909\\
81	0.444545454545455\\
91	0.427272727272727\\
101	0.42\\
111	0.402727272727273\\
121	0.406363636363636\\
131	0.404545454545455\\
141	0.393636363636364\\
151	0.381818181818182\\
161	0.384545454545455\\
171	0.38\\
181	0.382727272727273\\
191	0.372727272727273\\
201	0.369090909090909\\
211	0.365454545454545\\
221	0.35\\
231	0.373636363636364\\
241	0.361818181818182\\
251	0.370909090909091\\
261	0.369090909090909\\
271	0.365454545454545\\
281	0.355454545454545\\
291	0.345454545454545\\
301	0.349090909090909\\
311	0.337272727272727\\
321	0.342727272727273\\
331	0.331818181818182\\
341	0.336363636363636\\
351	0.35\\
361	0.34\\
371	0.322727272727273\\
381	0.314545454545455\\
391	0.324545454545455\\
401	0.314545454545455\\
411	0.314545454545455\\
421	0.313636363636364\\
431	0.323636363636364\\
441	0.313636363636364\\
451	0.310909090909091\\
461	0.312727272727273\\
471	0.309090909090909\\
481	0.294545454545455\\
491	0.301818181818182\\
501	0.307272727272727\\
511	0.290909090909091\\
521	0.298181818181818\\
531	0.298181818181818\\
541	0.303636363636364\\
551	0.29\\
561	0.285454545454545\\
571	0.296363636363636\\
581	0.295454545454545\\
591	0.282727272727273\\
601	0.291818181818182\\
611	0.292727272727273\\
621	0.269090909090909\\
631	0.283636363636364\\
641	0.29\\
651	0.282727272727273\\
661	0.286363636363636\\
671	0.286363636363636\\
681	0.283636363636364\\
691	0.275454545454545\\
701	0.287272727272727\\
711	0.292727272727273\\
721	0.285454545454545\\
731	0.279090909090909\\
741	0.262727272727273\\
751	0.258181818181818\\
761	0.266363636363636\\
771	0.278181818181818\\
781	0.273636363636364\\
791	0.273636363636364\\
801	0.265454545454545\\
811	0.27\\
821	0.26\\
831	0.259090909090909\\
841	0.258181818181818\\
851	0.26\\
861	0.268181818181818\\
871	0.258181818181818\\
881	0.260909090909091\\
891	0.26\\
901	0.267272727272727\\
911	0.264545454545455\\
921	0.257272727272727\\
931	0.246363636363636\\
941	0.236363636363636\\
951	0.245454545454545\\
961	0.248181818181818\\
971	0.248181818181818\\
981	0.245454545454545\\
991	0.231818181818182\\
1001	0.249090909090909\\
1011	0.241818181818182\\
1021	0.24\\
1031	0.227272727272727\\
1041	0.239090909090909\\
1051	0.232727272727273\\
1061	0.235454545454545\\
1071	0.232727272727273\\
1081	0.251818181818182\\
1091	0.253636363636364\\
1101	0.254545454545455\\
1111	0.252727272727273\\
1121	0.249090909090909\\
1131	0.24\\
1141	0.230909090909091\\
1151	0.243636363636364\\
1161	0.236363636363636\\
1171	0.223636363636364\\
1181	0.244545454545455\\
1191	0.220909090909091\\
1201	0.231818181818182\\
1211	0.244545454545455\\
1221	0.225454545454545\\
1231	0.245454545454545\\
1241	0.235454545454545\\
1251	0.240909090909091\\
1261	0.247272727272727\\
1271	0.225454545454545\\
1281	0.22\\
1291	0.229090909090909\\
1301	0.229090909090909\\
1311	0.231818181818182\\
1321	0.223636363636364\\
1331	0.225454545454545\\
1341	0.214545454545455\\
1351	0.224545454545455\\
1361	0.229090909090909\\
1371	0.220909090909091\\
1381	0.226363636363636\\
1391	0.23\\
1401	0.212727272727273\\
1411	0.234545454545455\\
1421	0.224545454545455\\
1431	0.222727272727273\\
1441	0.206363636363636\\
1451	0.202727272727273\\
1461	0.218181818181818\\
1471	0.210909090909091\\
1481	0.218181818181818\\
1491	0.21\\
1501	0.207272727272727\\
1511	0.216363636363636\\
1521	0.205454545454545\\
1531	0.202727272727273\\
1541	0.207272727272727\\
1551	0.205454545454545\\
1561	0.216363636363636\\
1571	0.213636363636364\\
1581	0.214545454545455\\
1591	0.219090909090909\\
1601	0.21\\
1611	0.217272727272727\\
1621	0.201818181818182\\
1631	0.202727272727273\\
1641	0.216363636363636\\
1651	0.200909090909091\\
1661	0.206363636363636\\
1671	0.214545454545455\\
1681	0.200909090909091\\
1691	0.196363636363636\\
1701	0.210909090909091\\
1711	0.191818181818182\\
1721	0.203636363636364\\
1731	0.208181818181818\\
1741	0.216363636363636\\
1751	0.220909090909091\\
1761	0.215454545454545\\
1771	0.198181818181818\\
1781	0.21\\
1791	0.200909090909091\\
1801	0.206363636363636\\
1811	0.202727272727273\\
1821	0.206363636363636\\
1831	0.208181818181818\\
1841	0.200909090909091\\
1851	0.197272727272727\\
1861	0.210909090909091\\
1871	0.198181818181818\\
1881	0.194545454545455\\
1891	0.202727272727273\\
1901	0.208181818181818\\
1911	0.214545454545455\\
1921	0.199090909090909\\
1931	0.213636363636364\\
1941	0.210909090909091\\
1951	0.199090909090909\\
1961	0.188181818181818\\
1971	0.200909090909091\\
1981	0.201818181818182\\
1991	0.205454545454545\\
};
\addlegendentry{\footnotesize CVV-Pro}
\end{axis}
\end{tikzpicture}%
 \\
        (c)& & (d) \\[-1em]
        & \setlength\figurewidth{3.75cm}
%
%
\definecolor{mycolor1}{rgb}{0.00000,0.44700,0.74100}%
\definecolor{mycolor2}{rgb}{0.85000,0.32500,0.09800}%
\definecolor{mycolor3}{rgb}{0.92900,0.69400,0.12500}%
\begin{tikzpicture}

\begin{axis}[%
width=1.2\figurewidth,
height=0.75\figurewidth,
at={(0\figurewidth,0\figurewidth)},
scale only axis,
xmin=0,
xmax=2000,
xlabel style={font=\color{white!15!black}},
xlabel={rounds},
ymin=0.05,
ymax=0.35,
ylabel style={font=\color{white!15!black}},
ylabel={projection time (s)},
axis background/.style={fill=white},
legend style={legend cell align=left, align=left, draw=white!15!black}
]
\addplot [color=mycolor1]
  table[row sep=crcr]{%
1	0.216187\\
11	0.106992\\
21	0.159891\\
31	0.111651\\
41	0.112754\\
51	0.1226\\
61	0.118676\\
71	0.107566\\
81	0.106775\\
91	0.123748\\
101	0.113387\\
111	0.132053\\
121	0.135728\\
131	0.130773\\
141	0.116277\\
151	0.106216\\
161	0.121399\\
171	0.120177\\
181	0.120902\\
191	0.121455\\
201	0.120374\\
211	0.122981\\
221	0.120612\\
231	0.120081\\
241	0.119833\\
251	0.120594\\
261	0.10517\\
271	0.110253\\
281	0.12863\\
291	0.129158\\
301	0.109966\\
311	0.118732\\
321	0.132002\\
331	0.128613\\
341	0.128393\\
351	0.109761\\
361	0.109608\\
371	0.11085\\
381	0.131082\\
391	0.108994\\
401	0.118019\\
411	0.127835\\
421	0.117608\\
431	0.117758\\
441	0.109865\\
451	0.119244\\
461	0.118042\\
471	0.097149\\
481	0.108531\\
491	0.118921\\
501	0.119854\\
511	0.10938\\
521	0.116574\\
531	0.107462\\
541	0.109862\\
551	0.108203\\
561	0.118112\\
571	0.119094\\
581	0.109932\\
591	0.110348\\
601	0.117358\\
611	0.109354\\
621	0.111374\\
631	0.121026\\
641	0.139233\\
651	0.120354\\
661	0.109169\\
671	0.098421\\
681	0.10761\\
691	0.117027\\
701	0.105759\\
711	0.131318\\
721	0.107792\\
731	0.10732\\
741	0.107298\\
751	0.117843\\
761	0.107926\\
771	0.128318\\
781	0.097248\\
791	0.108014\\
801	0.108734\\
811	0.118513\\
821	0.11131\\
831	0.10995\\
841	0.109423\\
851	0.117339\\
861	0.11101\\
871	0.123307\\
881	0.116328\\
891	0.125644\\
901	0.113936\\
911	0.108189\\
921	0.096133\\
931	0.120057\\
941	0.108382\\
951	0.118049\\
961	0.099176\\
971	0.127014\\
981	0.105373\\
991	0.107127\\
1001	0.097758\\
1011	0.109235\\
1021	0.109265\\
1031	0.099017\\
1041	0.120263\\
1051	0.102797\\
1061	0.119058\\
1071	0.119571\\
1081	0.1189\\
1091	0.119131\\
1101	0.119609\\
1111	0.108693\\
1121	0.09459\\
1131	0.131465\\
1141	0.109542\\
1151	0.103042\\
1161	0.121237\\
1171	0.129608\\
1181	0.108186\\
1191	0.11693\\
1201	0.107335\\
1211	0.107113\\
1221	0.132585\\
1231	0.120972\\
1241	0.085983\\
1251	0.112001\\
1261	0.108252\\
1271	0.096942\\
1281	0.11686\\
1291	0.108061\\
1301	0.119841\\
1311	0.120143\\
1321	0.119432\\
1331	0.119148\\
1341	0.117868\\
1351	0.117357\\
1361	0.108658\\
1371	0.094993\\
1381	0.096613\\
1391	0.109327\\
1401	0.117534\\
1411	0.112231\\
1421	0.119618\\
1431	0.12968\\
1441	0.120983\\
1451	0.108598\\
1461	0.105741\\
1471	0.116885\\
1481	0.109714\\
1491	0.108296\\
1501	0.103762\\
1511	0.109762\\
1521	0.116191\\
1531	0.098774\\
1541	0.115415\\
1551	0.119726\\
1561	0.109119\\
1571	0.11846\\
1581	0.12151\\
1591	0.119898\\
1601	0.118187\\
1611	0.108061\\
1621	0.117628\\
1631	0.110114\\
1641	0.119914\\
1651	0.106857\\
1661	0.106518\\
1671	0.110043\\
1681	0.120312\\
1691	0.11304\\
1701	0.110286\\
1711	0.121267\\
1721	0.111558\\
1731	0.108258\\
1741	0.117956\\
1751	0.107685\\
1761	0.12912\\
1771	0.116961\\
1781	0.122435\\
1791	0.130286\\
1801	0.120318\\
1811	0.128339\\
1821	0.096927\\
1831	0.128207\\
1841	0.108721\\
1851	0.116437\\
1861	0.129962\\
1871	0.110092\\
1881	0.109759\\
1891	0.106921\\
1901	0.118064\\
1911	0.118144\\
1921	0.107562\\
1931	0.108295\\
1941	0.119415\\
1951	0.107638\\
1961	0.129773\\
1971	0.126184\\
1981	0.118803\\
1991	0.109736\\
};
\addlegendentry{\footnotesize CVV-Pro}

\addplot [color=mycolor2]
  table[row sep=crcr]{%
1	0.303214\\
11	0.186491\\
21	0.21211\\
31	0.187101\\
41	0.187492\\
51	0.181682\\
61	0.17569\\
71	0.15955\\
81	0.156208\\
91	0.18439\\
101	0.156214\\
111	0.182638\\
121	0.180666\\
131	0.183378\\
141	0.181308\\
151	0.180867\\
161	0.181007\\
171	0.184244\\
181	0.18676\\
191	0.183097\\
201	0.187005\\
211	0.182184\\
221	0.188346\\
231	0.182213\\
241	0.182144\\
251	0.156103\\
261	0.156788\\
271	0.180736\\
281	0.155772\\
291	0.184464\\
301	0.159113\\
311	0.185195\\
321	0.184677\\
331	0.185948\\
341	0.188307\\
351	0.188598\\
361	0.179012\\
371	0.157068\\
381	0.185721\\
391	0.18624\\
401	0.182577\\
411	0.184055\\
421	0.185168\\
431	0.193026\\
441	0.18476\\
451	0.189466\\
461	0.176302\\
471	0.178351\\
481	0.186714\\
491	0.155119\\
501	0.191612\\
511	0.177102\\
521	0.182972\\
531	0.179447\\
541	0.167098\\
551	0.15741\\
561	0.181746\\
571	0.182912\\
581	0.184218\\
591	0.181154\\
601	0.196729\\
611	0.156103\\
621	0.187803\\
631	0.160037\\
641	0.187725\\
651	0.185675\\
661	0.184166\\
671	0.168734\\
681	0.156229\\
691	0.18324\\
701	0.180626\\
711	0.158359\\
721	0.182988\\
731	0.181916\\
741	0.182321\\
751	0.188384\\
761	0.181949\\
771	0.181034\\
781	0.182798\\
791	0.185113\\
801	0.184061\\
811	0.163378\\
821	0.154463\\
831	0.181002\\
841	0.185795\\
851	0.158673\\
861	0.175966\\
871	0.184589\\
881	0.190366\\
891	0.156086\\
901	0.184506\\
911	0.187005\\
921	0.158643\\
931	0.15758\\
941	0.169584\\
951	0.186782\\
961	0.183037\\
971	0.186282\\
981	0.184148\\
991	0.1868\\
1001	0.187958\\
1011	0.183889\\
1021	0.185985\\
1031	0.190157\\
1041	0.182477\\
1051	0.182059\\
1061	0.182983\\
1071	0.167663\\
1081	0.166903\\
1091	0.164078\\
1101	0.18349\\
1111	0.18212\\
1121	0.187209\\
1131	0.175025\\
1141	0.185203\\
1151	0.18535\\
1161	0.186239\\
1171	0.185386\\
1181	0.169655\\
1191	0.185265\\
1201	0.186994\\
1211	0.181212\\
1221	0.185771\\
1231	0.182801\\
1241	0.182861\\
1251	0.199313\\
1261	0.177161\\
1271	0.182826\\
1281	0.170799\\
1291	0.185325\\
1301	0.181773\\
1311	0.184256\\
1321	0.184705\\
1331	0.180657\\
1341	0.166678\\
1351	0.182698\\
1361	0.184372\\
1371	0.156665\\
1381	0.184232\\
1391	0.188078\\
1401	0.159731\\
1411	0.186389\\
1421	0.171295\\
1431	0.174559\\
1441	0.175184\\
1451	0.185209\\
1461	0.184339\\
1471	0.167616\\
1481	0.184884\\
1491	0.180532\\
1501	0.167389\\
1511	0.183456\\
1521	0.185778\\
1531	0.212695\\
1541	0.161387\\
1551	0.187342\\
1561	0.185315\\
1571	0.157904\\
1581	0.185177\\
1591	0.183782\\
1601	0.174378\\
1611	0.163512\\
1621	0.182616\\
1631	0.182893\\
1641	0.184625\\
1651	0.185846\\
1661	0.154173\\
1671	0.184661\\
1681	0.177164\\
1691	0.181348\\
1701	0.182148\\
1711	0.169503\\
1721	0.159005\\
1731	0.181075\\
1741	0.15806\\
1751	0.190604\\
1761	0.184989\\
1771	0.155866\\
1781	0.154685\\
1791	0.172824\\
1801	0.181239\\
1811	0.182588\\
1821	0.182432\\
1831	0.189375\\
1841	0.19917\\
1851	0.184231\\
1861	0.186887\\
1871	0.185645\\
1881	0.185215\\
1891	0.170848\\
1901	0.18277\\
1911	0.183702\\
1921	0.186556\\
1931	0.186795\\
1941	0.182853\\
1951	0.15841\\
1961	0.182245\\
1971	0.183443\\
1981	0.18244\\
1991	0.185031\\
};
\addlegendentry{\footnotesize OGD}

\end{axis}
\end{tikzpicture}%
 &&
	\setlength\figurewidth{3.75cm}
%
%
\definecolor{mycolor1}{rgb}{0.00000,0.44700,0.74100}%
\definecolor{mycolor2}{rgb}{0.85000,0.32500,0.09800}%
\begin{tikzpicture}

\begin{axis}[%
width=1.2\figurewidth,
height=0.75\figurewidth,
at={(0\figurewidth,0\figurewidth)},
scale only axis,
xmin=0,
xmax=2000,
xlabel style={font=\color{white!15!black}},
xlabel={rounds},
ymin=0,
ymax=40,
ylabel style={font=\color{white!15!black}},
ylabel={total projection time (s)},
axis background/.style={fill=white},
legend style={at={(0.03,0.97)}, anchor=north west, legend cell align=left, align=left, draw=white!15!black}
]
\addplot [color=mycolor1]
  table[row sep=crcr]{%
1	0.216187\\
11	0.323179\\
21	0.48307\\
31	0.594721\\
41	0.707475\\
51	0.830075\\
61	0.948751\\
71	1.056317\\
81	1.163092\\
91	1.28684\\
101	1.400227\\
111	1.53228\\
121	1.668008\\
131	1.798781\\
141	1.915058\\
151	2.021274\\
161	2.142673\\
171	2.26285\\
181	2.383752\\
191	2.505207\\
201	2.625581\\
211	2.748562\\
221	2.869174\\
231	2.989255\\
241	3.109088\\
251	3.229682\\
261	3.334852\\
271	3.445105\\
281	3.573735\\
291	3.702893\\
301	3.812859\\
311	3.931591\\
321	4.063593\\
331	4.192206\\
341	4.320599\\
351	4.43036\\
361	4.539968\\
371	4.650818\\
381	4.7819\\
391	4.890894\\
401	5.008913\\
411	5.136748\\
421	5.254356\\
431	5.372114\\
441	5.481979\\
451	5.601223\\
461	5.719265\\
471	5.816414\\
481	5.924945\\
491	6.043866\\
501	6.16372\\
511	6.2731\\
521	6.389674\\
531	6.497136\\
541	6.606998\\
551	6.715201\\
561	6.833313\\
571	6.952407\\
581	7.062339\\
591	7.172687\\
601	7.290045\\
611	7.399399\\
621	7.510773\\
631	7.631799\\
641	7.771032\\
651	7.891386\\
661	8.000555\\
671	8.098976\\
681	8.206586\\
691	8.323613\\
701	8.429372\\
711	8.56069\\
721	8.668482\\
731	8.775802\\
741	8.8831\\
751	9.000943\\
761	9.108869\\
771	9.237187\\
781	9.334435\\
791	9.442449\\
801	9.551183\\
811	9.669696\\
821	9.781006\\
831	9.890956\\
841	10.000379\\
851	10.117718\\
861	10.228728\\
871	10.352035\\
881	10.468363\\
891	10.594007\\
901	10.707943\\
911	10.816132\\
921	10.912265\\
931	11.032322\\
941	11.140704\\
951	11.258753\\
961	11.357929\\
971	11.484943\\
981	11.590316\\
991	11.697443\\
1001	11.795201\\
1011	11.904436\\
1021	12.013701\\
1031	12.112718\\
1041	12.232981\\
1051	12.335778\\
1061	12.454836\\
1071	12.574407\\
1081	12.693307\\
1091	12.812438\\
1101	12.932047\\
1111	13.04074\\
1121	13.13533\\
1131	13.266795\\
1141	13.376337\\
1151	13.479379\\
1161	13.600616\\
1171	13.730224\\
1181	13.83841\\
1191	13.95534\\
1201	14.062675\\
1211	14.169788\\
1221	14.302373\\
1231	14.423345\\
1241	14.509328\\
1251	14.621329\\
1261	14.729581\\
1271	14.826523\\
1281	14.943383\\
1291	15.051444\\
1301	15.171285\\
1311	15.291428\\
1321	15.41086\\
1331	15.530008\\
1341	15.647876\\
1351	15.765233\\
1361	15.873891\\
1371	15.968884\\
1381	16.065497\\
1391	16.174824\\
1401	16.292358\\
1411	16.404589\\
1421	16.524207\\
1431	16.653887\\
1441	16.77487\\
1451	16.883468\\
1461	16.989209\\
1471	17.106094\\
1481	17.215808\\
1491	17.324104\\
1501	17.427866\\
1511	17.537628\\
1521	17.653819\\
1531	17.752593\\
1541	17.868008\\
1551	17.987734\\
1561	18.096853\\
1571	18.215313\\
1581	18.336823\\
1591	18.456721\\
1601	18.574908\\
1611	18.682969\\
1621	18.800597\\
1631	18.910711\\
1641	19.030625\\
1651	19.137482\\
1661	19.244\\
1671	19.354043\\
1681	19.474355\\
1691	19.587395\\
1701	19.697681\\
1711	19.818948\\
1721	19.930506\\
1731	20.038764\\
1741	20.15672\\
1751	20.264405\\
1761	20.393525\\
1771	20.510486\\
1781	20.632921\\
1791	20.763207\\
1801	20.883525\\
1811	21.011864\\
1821	21.108791\\
1831	21.236998\\
1841	21.345719\\
1851	21.462156\\
1861	21.592118\\
1871	21.70221\\
1881	21.811969\\
1891	21.91889\\
1901	22.036954\\
1911	22.155098\\
1921	22.26266\\
1931	22.370955\\
1941	22.49037\\
1951	22.598008\\
1961	22.727781\\
1971	22.853965\\
1981	22.972768\\
1991	23.082504\\
};
\addlegendentry{\footnotesize CVV-Pro}

\addplot [color=mycolor2]
  table[row sep=crcr]{%
1	0.303214\\
11	0.489705\\
21	0.701815\\
31	0.888916\\
41	1.076408\\
51	1.25809\\
61	1.43378\\
71	1.59333\\
81	1.749538\\
91	1.933928\\
101	2.090142\\
111	2.27278\\
121	2.453446\\
131	2.636824\\
141	2.818132\\
151	2.998999\\
161	3.180006\\
171	3.36425\\
181	3.55101\\
191	3.734107\\
201	3.921112\\
211	4.103296\\
221	4.291642\\
231	4.473855\\
241	4.655999\\
251	4.812102\\
261	4.96889\\
271	5.149626\\
281	5.305398\\
291	5.489862\\
301	5.648975\\
311	5.83417\\
321	6.018847\\
331	6.204795\\
341	6.393102\\
351	6.5817\\
361	6.760712\\
371	6.91778\\
381	7.103501\\
391	7.289741\\
401	7.472318\\
411	7.656373\\
421	7.841541\\
431	8.034567\\
441	8.219327\\
451	8.408793\\
461	8.585095\\
471	8.763446\\
481	8.95016\\
491	9.105279\\
501	9.296891\\
511	9.473993\\
521	9.656965\\
531	9.836412\\
541	10.00351\\
551	10.16092\\
561	10.342666\\
571	10.525578\\
581	10.709796\\
591	10.89095\\
601	11.087679\\
611	11.243782\\
621	11.431585\\
631	11.591622\\
641	11.779347\\
651	11.965022\\
661	12.149188\\
671	12.317922\\
681	12.474151\\
691	12.657391\\
701	12.838017\\
711	12.996376\\
721	13.179364\\
731	13.36128\\
741	13.543601\\
751	13.731985\\
761	13.913934\\
771	14.094968\\
781	14.277766\\
791	14.462879\\
801	14.64694\\
811	14.810318\\
821	14.964781\\
831	15.145783\\
841	15.331578\\
851	15.490251\\
861	15.666217\\
871	15.850806\\
881	16.041172\\
891	16.197258\\
901	16.381764\\
911	16.568769\\
921	16.727412\\
931	16.884992\\
941	17.054576\\
951	17.241358\\
961	17.424395\\
971	17.610677\\
981	17.794825\\
991	17.981625\\
1001	18.169583\\
1011	18.353472\\
1021	18.539457\\
1031	18.729614\\
1041	18.912091\\
1051	19.09415\\
1061	19.277133\\
1071	19.444796\\
1081	19.611699\\
1091	19.775777\\
1101	19.959267\\
1111	20.141387\\
1121	20.328596\\
1131	20.503621\\
1141	20.688824\\
1151	20.874174\\
1161	21.060413\\
1171	21.245799\\
1181	21.415454\\
1191	21.600719\\
1201	21.787713\\
1211	21.968925\\
1221	22.154696\\
1231	22.337497\\
1241	22.520358\\
1251	22.719671\\
1261	22.896832\\
1271	23.079658\\
1281	23.250457\\
1291	23.435782\\
1301	23.617555\\
1311	23.801811\\
1321	23.986516\\
1331	24.167173\\
1341	24.333851\\
1351	24.516549\\
1361	24.700921\\
1371	24.857586\\
1381	25.041818\\
1391	25.229896\\
1401	25.389627\\
1411	25.576016\\
1421	25.747311\\
1431	25.92187\\
1441	26.097054\\
1451	26.282263\\
1461	26.466602\\
1471	26.634218\\
1481	26.819102\\
1491	26.999634\\
1501	27.167023\\
1511	27.350479\\
1521	27.536257\\
1531	27.748952\\
1541	27.910339\\
1551	28.097681\\
1561	28.282996\\
1571	28.4409\\
1581	28.626077\\
1591	28.809859\\
1601	28.984237\\
1611	29.147749\\
1621	29.330365\\
1631	29.513258\\
1641	29.697883\\
1651	29.883729\\
1661	30.037902\\
1671	30.222563\\
1681	30.399727\\
1691	30.581075\\
1701	30.763223\\
1711	30.932726\\
1721	31.091731\\
1731	31.272806\\
1741	31.430866\\
1751	31.62147\\
1761	31.806459\\
1771	31.962325\\
1781	32.11701\\
1791	32.289834\\
1801	32.471073\\
1811	32.653661\\
1821	32.836093\\
1831	33.025468\\
1841	33.224638\\
1851	33.408869\\
1861	33.595756\\
1871	33.781401\\
1881	33.966616\\
1891	34.137464\\
1901	34.320234\\
1911	34.503936\\
1921	34.690492\\
1931	34.877287\\
1941	35.06014\\
1951	35.21855\\
1961	35.400795\\
1971	35.584238\\
1981	35.766678\\
1991	35.951709\\
};
\addlegendentry{\footnotesize OGD}

\end{axis}
\end{tikzpicture}%
     \end{tabular}
    \vspace{-1.1em}
    \caption{The figure contrasts CVV-Pro and OGD by comparing the resulting regret (a) and execution time (c,d).
    Panel (b) shows how the number of violated constraints evolves over time.}
	\label{fig:OGD_vs_CVVPro}
\end{figure}

\newpage
\section{Proof of Theorem \ref{thm:main}}\label{app:thm_main}

In this section, we consider an online optimization problem with time-invariant constraints and a bounded iterate assumption.
The bounded iterate assumption will be removed subsequently in Section~\ref{sec:algorithmic_result}, which however, will require a more complex analysis.

We restate Theorem~\ref{thm:main} below for the convenience of the reader.
 
\begin{theorem}[Structural]
Suppose Assumption~\ref{as:fg} holds and in addition $x_{t}\in\mathcal{B}_R$ for all $t\in\{1,\dots,T\}$.
Then, on input $\alpha = L_{\mathcal{F}}/R$, Algorithm \ref{alg:cvv-pro} with step sizes
$\eta_{t}=\frac{1}{\alpha\sqrt{t}}$ guarantees the following for all $T\geq1$: 
 \vspace{-5pt}
 \begin{itemize}[noitemsep, topsep=0pt, leftmargin=1pc]
 \item[]\textbf{(regret)}\quad\quad $\sum_{t=1}^{T}f_{t}(x_{t})-\min_{x\in\mathcal{C}}\sum_{t=1}^{T}f_{t}(x)\leq 18L_{\mathcal{F}}R\sqrt{T}$;
 \item[]\textbf{(feasibility)}\, $g_{i}(x_{t})\geq-8\left[\frac{L_{\mathcal{G}}}{R}+2\beta_{\mathcal{G}}\right]\frac{R^{2}}{\sqrt{t}}$,\quad for all $t\in\{1,\dots,T\}$ and $i\in\{1,\dots,m\}$.
 \end{itemize}
\end{theorem}

The rest of this section is devoted to proving the preceding statement.
 
\subsection{Structural Properties}\label{app:structural_propertires}

\begin{lemma}\label{lem:velocity} 
    Suppose $g_{i}$ is concave for every $i\in\{1,\dots,m\}$.
    Then, for any $\alpha>0$ and all $x\in\mathcal{C}$ the following holds
    \[
    \max_{t\geq0}\lVert v_{t}\rVert\leq\alpha\lVert x - x_{t}\rVert 
    + 2\lVert\nabla f_{t}(x_{t})\rVert.
    \]
    In particular, when $f_t$ satisfies $\Vert \nabla f_{t}(z) \Vert \leq L_{\mathcal{F}}$ for all $z\in\mathcal{B}_{cR}$, it follows that
    $\lVert v_{t}\rVert\leq(c+1)\alpha R+2L_{\mathcal{F}}$ for any $x\in\mathcal{B}_{R}$, $x_{t}\in\mathcal{B}_{cR}$, and $c>0$.
\end{lemma}
\begin{proof}
    By Claim \ref{clm:concave}, we have $\alpha(x-x_{t})\in V_{\alpha}(x_{t})$ for every $x\in\mathcal{C}$. 
    Combining the triangle inequality with the fact that $v_{t}$ is an optimal solution of the velocity projection problem in Step~\ref{eq:velocity_projection}, yields
    \begin{eqnarray*}
        \lVert v_{t}\rVert-\lVert\nabla f_{t}(x_{t})\rVert & \leq & \lVert v_{t}+\nabla f_{t}(x_{t})\rVert\\
        & \leq & \lVert\alpha(x - x_{t})+\nabla f_{t}(x_{t})\rVert\\
        & \leq & \alpha\lVert x - x_{t}\rVert+\lVert\nabla f_{t}(x_{t})\rVert.
    \end{eqnarray*}
    Using $x\in\mathcal{B}_{R}$, $x_{t}\in\mathcal{B}_{cR}$ and $\lVert\nabla f_t(x_t)\rVert\leq L_{\mathcal{F}}$, we conclude
    \[
    \lVert v_{t}\rVert\leq\alpha\lVert x-x_{t}\rVert+2\lVert\nabla f_{t}(x_{t})\rVert\leq(c+1)\alpha R+2L_{\mathcal{F}}.
    \]
\end{proof}

\subsection{Cost Regret}
\begin{lemma}\label{lem:OGD}
    Suppose Assumption~\ref{as:fg} holds and $x_{t}\in\mathcal{B}_{cR}$ for all $t\in\{1,\dots,T\}$ with $c\in(0,4]$.
    Let $d\geq0$ be a constant.
    Then, Algorithm~\ref{alg:cvv-pro} applied with $\alpha=L_{\mathcal{F}}/R$ and step sizes $\eta_{t}=\frac{1}{\alpha\sqrt{t+d}}$, guarantees the following for all $T\geq1$: 
    \[
    R_{T}=\sum_{t=1}^{T}f_{t}(x_{t})-\min_{x^{\star}\in\mathcal{C}}\sum_{t=1}^{T}f_{t}(x^{\star})\leq\sqrt{d+1}\left[(c+3)^{2}+\frac{1}{2}(c+1)^{2}\right]L_{\mathcal{F}}R\sqrt{T}.
    \]
    In particular, for $c=1$ and $d=0$ we have $R_{T}\leq 18L_{\mathcal{F}}R\sqrt{T}$.
\end{lemma}
\begin{proof}
    We denote an optimal decision in hindsight by $x^{\star}\in\argmin_{x\in\mathcal{C}}\sum_{t=1}^{T}f_{t}(x)$.
    For any points $x^{\star}$, $x_t$ we have $f_{t}(x_{t})-f_{t}(x^{\star})\leq[\nabla f_{t}(x_{t})]^{\T}(x_{t}-x^{\star})$, since $f_t$ is convex. 
    Summing over the number of rounds $t$ results in
    \[
    \sum_{t=1}^{T}f_{t}(x_{t})-f_{t}(x^{\star}) \leq \sum_{t=1}^{T} [\nabla f_{t}(x_{t})]^{\T}(x_{t}-x^{\star}).
    \]
    We proceed by upper bounding the expression $[\nabla f_{t}(x_{t})]^{\T}(x_{t}-x^{\star})$.
    Using $x_{t+1}=x_{t}+\eta_{t}v_{t}$ and $v_{t}=r_{t}-\nabla f_{t}(x_{t})$, we have
    \begin{eqnarray*}
        \lVert x_{t+1}-x^{\star}\rVert^{2}&=&\lVert x_{t}+\eta_{t}\left(r_{t}-\nabla f_{t}(x_{t})\right)-x^{\star}\rVert^{2}\\&=&\lVert x_{t}-x^{\star}\rVert^{2}+\eta_{t}^{2}\lVert r_{t}-\nabla f_{t}(x_{t})\rVert^{2}+2\eta_{t}\left[r_{t}-\nabla f_{t}(x_{t})\right]^{\T}(x_{t}-x^{\star}).
    \end{eqnarray*}
    Then, Lemma~\ref{lem:main} gives $r_{t}^{\T}(x_{t}-x^{\star})\leq0$ and thus
    \begin{eqnarray*}
        [\nabla f_{t}(x_{t})]^{\T}(x_{t}-x^{\star}) & = & r_{t}^{\T}(x_{t}-x^{\star})+\frac{\lVert x_{t}-x^{\star}\rVert^{2}-\lVert x_{t+1}-x^{\star}\rVert^{2}}{2\eta_{t}}+\frac{\eta_{t}}{2}\lVert v_{t}\rVert^{2}\\
        \\
        & \leq & \frac{\lVert x_{t}-x^{\star}\rVert^{2}-\lVert x_{t+1}-x^{\star}\rVert^{2}}{2\eta_{t}}+\frac{\eta_{t}}{2}\lVert v_{t}\rVert^{2}.
    \end{eqnarray*}

    Since $x^{\star}\in\mathcal{B}_{R}$ and $x_{t}\in\mathcal{B}_{cR}$ for all $t\in\{1,\dots,T\}$, by
    Lemma \ref{lem:velocity} it follows for all $t\in\{1,\dots,T\}$ that 
    \begin{equation}\label{eq:velocity}
	\lVert v_{t}\rVert\leq(c+1)\alpha R+2L_{\mathcal{F}}=(c+3)L_{\mathcal{F}}=:\mathcal{V}_{\alpha}.
    \end{equation}
    Summing over the whole sequence, using the fact that $\eta_{t}=\frac{1}{\alpha\sqrt{t+d}}$ is a decreasing positive sequence and applying Claim~\ref{clm:series}, $x^{\star}\in\mathcal{B}_{R}$, $x_{t}\in\mathcal{B}_{cR}$, and \eqref{eq:velocity}, yields
    \begin{eqnarray*}
        2\sum_{t=1}^{T}[\nabla f_{t}(x_{t})]^{\T}(x_{t}-x^{\star})&\leq&\sum_{t=1}^{T}\frac{\lVert x_{t}-x^{\star}\rVert^{2}-\lVert x_{t+1}-x^{\star}\rVert^{2}}{\eta_{t}}+\eta_{t}\lVert v_{t}\rVert^{2}\\&\leq&\mathcal{V}_{\alpha}^{2}\left(\sum_{t=1}^{T}\eta_{t}\right)+\frac{(c+1)^{2}R^{2}}{\eta_{T}}\\&\leq&(c+3)^{2}L_{\mathcal{F}}^{2}\frac{2}{\alpha}\sqrt{T+d}+(c+1)^{2}L_{\mathcal{F}}R\sqrt{T+d}\\&=&\left[2(c+3)^{2}+(c+1)^{2}\right]L_{\mathcal{F}}R\sqrt{T+d},
    \end{eqnarray*}
    where last inequality uses
    \[
    \sum_{t=1}^{T}\eta_{t}=\frac{1}{\alpha}\sum_{t=1}^{T}\frac{1}{\sqrt{t+d}}<\frac{1}{\alpha}\sum_{t=1}^{T+d}\frac{1}{\sqrt{t}}\leq\frac{2}{\alpha}\sqrt{T+d}.
    \]
    The statement follows by combining the fact that $\sqrt{T+d}\leq\sqrt{d+1}\sqrt{T}$ for any $d\geq0$ and all $T\geq1$, and
    \[
    \sum_{t=1}^{T}f_{t}(x_{t})-f_{t}(x^{\star})\leq\sum_{t=1}^{T}[\nabla f_{t}(x_{t})]^{\T}(x_{t}-x^{\star})\leq\sqrt{d+1}\left[(c+3)^{2}+\frac{1}{2}(c+1)^{2}\right]L_{\mathcal{F}}R\sqrt{T}.
    \]
\end{proof}

\begin{claim}[Series] \label{clm:series} 
    For any positive sequence $\{a_{t}\}_{t=1}^{T+1}$ and
    any decreasing positive sequence $\{\eta_{t}\}_{t=1}^{T}$, it holds that
    \[
        \sum_{t=1}^{T}\frac{a_{t}-a_{t+1}}{\eta_{t}}\leq\frac{A}{\eta_{T}},\quad\text{where}\mbox{\quad}A:=\max_{t=\{1,\dots,T\}}a_{t}.
    \]
\end{claim}
\begin{proof}
    Observe that
    \begin{eqnarray*}
        \sum_{t=1}^{T}\frac{a_{t}-a_{t+1}}{\eta_{t}}&=&\frac{a_{1}-a_{2}}{\eta_{1}}+\frac{a_{2}-a_{3}}{\eta_{2}}+\frac{a_{3}-a_{4}}{\eta_{3}}+\cdots+\frac{a_{T}-a_{T+1}}{\eta_{T}}\\&=&\frac{a_{1}}{\eta_{1}}-\frac{a_{T+1}}{\eta_{T}}+\sum_{i=2}^{\T}a_{i}\left(\frac{1}{\eta_{i}}-\frac{1}{\eta_{i-1}}\right)\\&\leq&\frac{A}{\eta_{T}},
    \end{eqnarray*}
    where the last inequality follows by
    \[
    \sum_{i=2}^{\T}a_{i}\left(\frac{1}{\eta_{i}}-\frac{1}{\eta_{i-1}}\right)\leq A\sum_{i=2}^{\T}\left(\frac{1}{\eta_{i}}-\frac{1}{\eta_{i-1}}\right)=A\left(\frac{1}{\eta_{T}}-\frac{1}{\eta_{1}}\right)\leq\frac{A}{\eta_{T}}-\frac{a_{1}}{\eta_{1}}.
    \]
\end{proof}

\subsection{Convergence Rate of Constraint Violations}

\begin{lemma}[Convergence Rate of Constraint Violations]\label{lem:feasibility_convergence}
    Suppose Assumption~\ref{as:fg} holds and $\{x_t\}_{t\geq1}\in\mathcal{B}_{cR}$ with $x_1\in\mathcal{B}_{R}$ and $c\in(0,4]$. 
    Then, for any $\alpha>0$ and $d\geq0$, step sizes $\eta_t=1/(\alpha\sqrt{t+d})$ and $\mathcal{V}_{\alpha}>0$ such that $\Vert v_{t}\Vert\leq\mathcal{V_{\alpha}}$ for all $t\geq1$, it follows for every $i\in\{1,\dots,m\}$ and $t\geq1$ that 
    \[
    g_{i}(x_{t}) \geq -c_{1}\eta_{t},
    \]
    where
    \[    c_{1}=\mathcal{V}_{\alpha}\left[2L_{\mathcal{G}}+\frac{\beta_{G}\mathcal{V}_{\alpha}}{\alpha}\right]+\mathcal{Z}_{d}\quad\text{and}\quad \mathcal{Z}_{d}=\left(1-\frac{1}{\sqrt{d+1}}\right)\sqrt{d+2}\left[\frac{L_{\mathcal{G}}}{R}+\beta_{\mathcal{G}}\right]2\alpha R^{2}.
    \]
    In particular, when Assumption~\ref{as:fg} holds, $\{x_t\}_{t\geq1}\in\mathcal{B}_{R}$, $\alpha=L_{\mathcal{F}}/R$ and $d=0$, it follows that
    \[
        g_{i}(x_{t})\geq-8\left[\frac{L_{\mathcal{G}}}{R}+2\beta_{G}\right]\frac{R^{2}}{\sqrt{t}}\quad\text{for all}\,t\geq1.
    \]
\end{lemma}
\begin{proof}
    The proof is by induction on $t$. 
    We start with the base case $t=1$.
    The proof proceeds by case distinction.
    
    \textbf{Case 1.} Suppose $i\in \{1,\dots,m\}\backslash I(x_{1})$, i.e., $g_{i}(x_{1})>0$. 
    Then, by Claim~\ref{clm:const_violation} Part ii) we have
    \[
    g_{i}(x_{2})\geq-\eta_{2}\mathcal{V}_{\alpha}\left[2L_{\mathcal{G}}+\frac{\mathcal{V}_{\alpha}\beta_{\mathcal{G}}}{\alpha\sqrt{1+d}}\right]\geq-c_{1}\eta_{2}.
    \] 
    \textbf{Case 2.} Suppose $i\in I(x_{1})$, i.e., $g_{i}(x_{1})\leq0$. 
    By combining $x_1\in\mathcal{B}_{R}$ and $g_i$ is concave $\beta_{\mathcal{G}}$-smooth, it follows for every $x\in\mathcal{C}\subseteq\mathcal{B}_{R}$ that
    \begin{eqnarray*}
        g_{i}(x_{1})&\geq&g_{i}(x)+\nabla g_{i}(x)^{T}(x_{1}-x)-\frac{\beta_{\mathcal{G}}}{2}\Vert x_{1}-x\Vert^{2}\\&\geq&-2L_{\mathcal{G}}R-2\beta_{\mathcal{G}}R^{2}\\&=&-\eta_{1}\sqrt{d+1}\left[\frac{L_{\mathcal{G}}}{R}+\beta_{\mathcal{G}}\right]2\alpha R^{2}\geq-c_{1}\eta_{2}.
    \end{eqnarray*}
    Using $\eta_{t}=1/(\alpha\sqrt{t+d})$, $\eta_{1}/\eta_{2}\leq\sqrt{2}$ and $\eta_{1}^{2}\frac{\mathcal{V}_{\alpha}^{2}\beta_{\mathcal{G}}}{2}\leq\eta_{2}^{2}\mathcal{V}_{\alpha}^{2}\beta_{\mathcal{G}}=\eta_{2}\frac{\mathcal{V}_{\alpha}^{2}\beta_{\mathcal{G}}}{\alpha\sqrt{d+2}}$, it follows by Claim \ref{clm:const_violation} Part i) that
    \begin{eqnarray*}
        g_{i}(x_{2})&\geq&(1-\alpha\eta_{1})g_{i}(x_{1})-\eta_{1}^{2}\frac{\mathcal{V}_{\alpha}^{2}\beta_{\mathcal{G}}}{2}\\&\geq&-\eta_{2}\left[\left(1-\frac{1}{\sqrt{d+1}}\right)\sqrt{d+2}\left[\frac{L_{\mathcal{G}}}{R}+\beta_{\mathcal{G}}\right]2\alpha R^{2}+\frac{\mathcal{V}_{\alpha}^{2}\beta_{\mathcal{G}}}{\alpha\sqrt{d+2}}\right]\geq-c_{1}\eta_{2}.
    \end{eqnarray*}
    Our inductive hypothesis is $g_{i}(x_{t})\geq-c_{1}\eta_{t}$ for all $i$.
    We now show that it holds for $t+1$. 
    
    \textbf{Case 1.} Suppose $i\in \{1,\dots,m\}\backslash I(x_{1})$, i.e., $g_{i}(x_{t})>0$.
    Then by Claim~\ref{clm:const_violation} Part ii)
    \[
    g_{i}(x_{t+1})\geq-\eta_{t+1}\mathcal{V}_{\alpha}\left[2L_{\mathcal{G}}+\frac{\beta_{\mathcal{G}}\mathcal{V}_{\alpha}}{\alpha\sqrt{d+1}}\right]\geq-c_{1}\eta_{t+1}.
    \]

    \textbf{Case 2.} Suppose $i\in I(x_{t})$, i.e., $g_{i}(x_{t})\leq0$.
    Combining Claim~\ref{clm:const_violation} Part ii) and the inductive
    hypothesis we have
    \begin{eqnarray*}
        g_{i}(x_{t+1})&\geq&(1-\alpha\eta_{t})g_{i}(x_{t})-\eta_{t}^{2}\frac{\mathcal{V}_{\alpha}^{2}\beta_{\mathcal{G}}}{2}\\&\geq&-c_{1}\eta_{t}+c_{1}\alpha\eta_{t}^{2}-\eta_{t}^{2}\frac{\mathcal{V}_{\alpha}^{2}\beta_{\mathcal{G}}}{2}\\&=&-c_{1}\eta_{t+1}+c_{1}\eta_{t+1}-c_{1}\eta_{t}+c_{1}\alpha\eta_{t}^{2}-\eta_{t}^{2}\frac{\mathcal{V}_{\alpha}^{2}\beta_{\mathcal{G}}}{2}\\&=&-c_{1}\eta_{t+1}+c_{1}\eta_{t}\left[\frac{\eta_{t+1}}{\eta_{t}}-1+\alpha\eta_{t}-\eta_{t}\frac{\mathcal{V}_{\alpha}^{2}\beta_{\mathcal{G}}}{2c_{1}}\right].
    \end{eqnarray*}
    Since $c_{1}\eta_{t}>0$, it suffices to show that 
    \begin{equation}
    \alpha-\frac{\eta_{t}-\eta_{t+1}}{\eta_{t}^{2}}\geq\frac{\mathcal{V}_{\alpha}^{2}\beta_{\mathcal{G}}}{c_{1}}\label{eq:c_lb}
    \end{equation}
    or equivalently (using $\eta_{t}=\frac{1}{\alpha\sqrt{t+d}}$
    for $t\geq1$)
    \[
        \alpha-\alpha\sqrt{\frac{t+d}{t+d+1}}\left(\sqrt{t+d+1}-\sqrt{t+d}\right)\geq\frac{\mathcal{V}_{\alpha}^{2}\beta_{\mathcal{G}}}{2c_{1}}.
    \]
    Straightforward checking shows that $\max_{t\geq1}\sqrt{\frac{t}{t+1}}\left(\sqrt{t+1}-\sqrt{t}\right)<\frac{1}{3}$.
    Hence, inequality \eqref{eq:c_lb} is implied for $c_1\geq\beta_{\mathcal{G}}\mathcal{V}_{\alpha}^{2}/\alpha$ and thus $g_{i}(x_{t+1})\geq-c_1\eta_{t+1}$.
    
    Furthermore, for $c=1$ and $\alpha=L_{\mathcal{F}}/R$, by Lemma~\ref{lem:velocity}, we can set $\mathcal{V}_{\alpha}=4L_{\mathcal{F}}$.
    Then, for $d=0$ we have $g_{i}(x_{t})\geq-8\left[\frac{L_{\mathcal{G}}}{R}+2\beta_{\mathcal{G}}\right]\frac{R^{2}}{\sqrt{t}}$ for all $t\geq1$.
\end{proof}

\begin{claim}[Constraint Violation]\label{clm:const_violation}
    Suppose $g_{i}$ is concave, $\beta_{\mathcal{G}}$-smooth and satisfies $\Vert \nabla g_{i}(x) \Vert \leq L_{\mathcal{G}}$ for all $x\in\mathcal{B}_{cR}$ and $i\in\{1,\dots,m\}$, where $c>0$ is a constant.
    Suppose further that there exists a constant $\mathcal{V}_{\alpha}>0$ such that $x_t\in\mathcal{B}_{cR}$ and $\Vert v_{t}\Vert\leq\mathcal{V_{\alpha}}$, for all $t\geq1$. 
    Then, for all $t\geq1$ we have
    
    i) $g_{i}(x_{t+1})\geq(1-\alpha\eta_{t})g_{i}(x_{t})-\eta_{t}^{2}\mathcal{V}_{\alpha}^{2}\beta_{\mathcal{G}}/2$\quad
    for every $i\in I(x_{t})$;
    
    ii) $g_{i}(x_{t+1})\geq-\eta_{t+1}\mathcal{V}_{\alpha}\left[2L_{\mathcal{G}}+\mathcal{V}_{\alpha}\beta_{\mathcal{G}}/(\alpha\sqrt{1+d})\right]$\quad
    for every $i\in\{1,\dots,m\}\backslash I(x_{t})$.
\end{claim}
\begin{proof}
    The proof proceeds by case distinction.
    
    \textbf{Case 1.} Suppose $i\in I(x_{t})$, i.e., $g_{i}(x_{t})\leq0$.
    By combining the facts that $g_{i}$ is concave and $\beta_{\mathcal{G}}$-smooth, $x_{t+1}=x_{t}+\eta_{t}v_{t}$
    and $[\nabla g_{i}(x_{t})]^{\T}v_{t}\geq-\alpha g_{i}(x_{t})$, it
    follows that
    \begin{eqnarray}\label{eq:LB_gi}
    g_{i}(x_{t+1}) & \geq & g_{i}(x_{t})+[\nabla g_{i}(x_{t})]^{\T}[x_{t+1}-x_{t}]-\frac{\beta_{\mathcal{G}}}{2}\lVert x_{t+1}-x_{t}\rVert_{2}^{2}\nonumber \\
    & \geq & (1-\alpha\eta_{t})g_{i}(x_{t})-\eta_{t}^{2}\mathcal{V}_{\alpha}^{2}\frac{\beta_{\mathcal{G}}}{2}.
    \end{eqnarray}

    \textbf{Case 2.} Suppose $i\not\in I(x_{t})$, i.e., $g_{i}(x_{t})>0$.
    Using $\Vert \nabla g_{i}(x) \Vert \leq L_{\mathcal{G}}$ for $x_t\in\mathcal{B}_{cR}$,
    we have
    \[
    [\nabla g_{i}(x_{t})]^{\T}[x_{t+1}-x_{t}]\leq\lVert\nabla g_{i}(x_{t})\rVert\lVert x_{t+1}-x_{t}\rVert\leq\eta_{t}L_{\mathcal{G}}\mathcal{V}_{\alpha}.
    \]
    Hence,
    \begin{eqnarray*}
        g_{i}(x_{t+1})&\geq&g_{i}(x_{t})+[\nabla g_{i}(x_{t})]^{\T}[x_{t+1}-x_{t}]-\frac{\beta_{\mathcal{G}}}{2}\lVert x_{t+1}-x_{t}\rVert_{2}^{2}\\&\geq&-\eta_{t}L_{\mathcal{G}}\mathcal{V}_{\alpha}-\eta_{t}^{2}\mathcal{V}_{\alpha}^{2}\frac{\beta_{\mathcal{G}}}{2}\\&=&-\eta_{t+1}\frac{\eta_{t}}{\eta_{t+1}}\mathcal{V}_{\alpha}\left[L_{\mathcal{G}}+\frac{\eta_{t}}{2}\mathcal{V}_{\alpha}\beta_{\mathcal{G}}\right]\\&>&-\eta_{t+1}\mathcal{V}_{\alpha}\left[2L_{\mathcal{G}}+\frac{\mathcal{V}_{\alpha}\beta_{\mathcal{G}}}{\alpha\sqrt{1+d}}\right],
    \end{eqnarray*}
    where the last inequality follows by $\eta_{t}\leq\eta_{1}=1/(\alpha\sqrt{1+d})$ and
    \[
        \max_{\ell\geq1}\frac{\eta_{\ell}}{\eta_{\ell+1}}\leq\max_{\ell\geq1}\sqrt{\frac{\ell+1}{\ell}}=\sqrt{2}.
    \]
\end{proof}	

\newpage
\section{Guaranteeing a Bounded Decision Sequence}\label{sec:algorithmic_result}

We now show that the second assumption in Theorem~\ref{thm:main}, namely, ``$x_{t}\in\mathcal{B}_{R}$ for all $t\in\{1,\dots,T\}$'' can be enforced algorithmically. 
We achieve this by introducing an additional hypersphere constraint $g_{m+1}(x_{t})=\frac{1}{2}[R^{2}-\Vert x_{t}\Vert^{2}]$ that attracts the decision sequence $\{x_{t}\}_{t\geq1}$ to a hypersphere $\mathcal{B}_{R}$ and guarantees that it always stays inside a hypersphere $\mathcal{B}_{4R}$ with a slightly larger radius.
Technically, we modify the velocity polyhedron in Step 3 of Algorithm~\ref{alg:cvv-pro} as follows:
$V_{\alpha}^{\prime}(x_{t})=V_{\alpha}(x_{t})$ if $\Vert x\Vert\leq R$, and otherwise
\[
V_{\alpha}^{\prime}(x_{t})=\{v\in V_{\alpha}(x_{t}) ~|~[\nabla g_{m+1}(x_{t})]^{\T}v\geq-\alpha g_{m+1}(x_{t})\}.
\]

We are now ready to state our main algorithmic result for the setting of time-invariant constraints.

\begin{theorem}[Algorithm]\label{thm:main_full} 
Suppose Assumption~\ref{as:fg} holds.
Then, on input $R,L_{\mathcal{F}}>0$, $\alpha=L_{\mathcal{F}}/R$ and $x_{1}\in\mathcal{B}_{R}$, Algorithm~\ref{alg:cvv-pro} with augmented velocity polyhedron $V_{\alpha}^{\prime}(\cdot)$ and step sizes $\eta_{t}=\frac{1}{\alpha\sqrt{t+15}}$ guarantees the following for all $T\geq1$:
\begin{itemize}[noitemsep, topsep=0pt, leftmargin=1pc]
\item[]\textbf{(regret)} \quad\, $\sum_{t=1}^{T}f_{t}(x_{t})-\min_{x^{\star}\in\mathcal{C}}\sum_{t=1}^{T}f_{t}(x^{\star})\leq 246L_{\mathcal{F}}R\sqrt{T}$;

\item[]\textbf{(feasibility)} $g_{i}(x_{t})\geq -21\big[\frac{L_{\mathcal{G}}}{R}+3\beta_{G}\big]\frac{R^{2}}{\sqrt{t+15}}$,\quad for all $t\in\{1,\dots,T\}$ and $i\in\{1,\dots,m\}$;

\item[]\textbf{(attraction)} $g_{m+1}(x_{t})\geq-27\frac{R^{2}}{\sqrt{t+15}}$ for all $t\in\{1,\dots,T\}$.
\end{itemize}
In addition, $\Vert x_{t}\Vert \leq 4R$ and $\Vert v_{t} \Vert\leq 7L_{\mathcal{F}}$, for all $t\geq1$.
\end{theorem}

To ensure convergence of the hypersphere constraint $-\min\{g_{m+1}(x_{t}),0\}$ at a rate of $\mathcal{O}(1/\sqrt{t})$, we use an inductive argument similar to Lemma~\ref{lem:feasibility_convergence}.
We note that compared to the simplified setting of Appendix~\ref{app:thm_main}, our analysis requires an additional refined inductive argument, which is summarized in Lemma~\ref{lem:gmp_all}.

\subsection{Hypersphere constraint}
\begin{definition}
    We consider the following hypersphere constraint, parameterized by $R>0$, 
    \[
        g_{m+1}(x)=\frac{1}{2}[R^{2}-\Vert x_{t}\Vert^{2}].
    \]
    By construction, $g_{m+1}$ is concave and $1$-smooth.
\end{definition}

\begin{claim}\label{clm:velocity_ub_v1} 
Suppose $g_i$ is concave for every $i\in\{1,\dots,m\}$ such that $\mathcal{C}\subseteq\mathcal{B}_{R}$ and $f_t$ is convex such that $\Vert\nabla f_t(x)\Vert\leq L_\mathcal{F}$ for all $x\in \mathcal{B}_{cR}$, where $c>0$ is a constant.
Then for any decision $x_t\in\mathcal{B}_{cR}$, it holds that
\[
\Vert v_{t}\Vert\leq\alpha\Vert x_{t}\Vert+\left(\alpha R+2L_{\mathcal{F}}\right)\qquad\text{and}\qquad\frac{1}{2}\Vert v_{t}\Vert^{2}<-2\alpha^{2}g_{m+1}(x)+\left[\alpha^{2}R^{2}+\left(\alpha R+2L_{\mathcal{F}}\right)^{2}\right].
\]
\end{claim}
\begin{proof}
Due to the fact that $g_{m+1}$ and $g_{i}$ are concave for every $i\in\{1,\dots,m\}$, it follows by Lemma~\ref{lem:velocity} that
\begin{eqnarray*}
\Vert v_{t}\Vert&\leq&2\Vert\nabla f_{t}(x_{t})\Vert+\alpha\Vert x^{\star}-x_{t}\Vert\\&\leq&\alpha\Vert x_{t}\Vert+\alpha R+2L_{\mathcal{F}}.
\end{eqnarray*}
Further, by definition of $g_{m+1}(x)$ we have
\begin{eqnarray*}
\frac{1}{2}\Vert v_{t}\Vert^{2}&\leq&\frac{1}{2}\left[\alpha\Vert x_{t}\Vert+\left(\alpha R+2L_{\mathcal{F}}\right)\right]^{2}\\&\leq&\alpha^{2}\Vert x_{t}\Vert^{2}+\left(\alpha R+2L_{\mathcal{F}}\right)^{2}\\&=&-2\alpha^{2}g_{m+1}(x)+\left[\alpha^{2}R^{2}+\left(\alpha R+2L_{\mathcal{F}}\right)^{2}\right].
\end{eqnarray*}
\end{proof}

\begin{claim}\label{clm:gmp} 
    Suppose the assertions in Claim~\ref{clm:velocity_ub_v1} hold.
    Let the step sizes be $\{\eta_{t}=\frac{1}{\alpha\sqrt{t+15}}\}_{t\geq1}$ and $\alpha=L_{\mathcal{F}}/R$.
    Then, we have
    
    i) If $g_{m+1}(x_{t})>0$ then $g_{m+1}(x_{t+1})\geq-\eta_{t}\cdot6L_{\mathcal{F}}R$; and
    
    ii) If $g_{m+1}(x_{t})\leq0$ then $g_{m+1}(x_{t+1})\geq\big(1-\frac{\alpha}{2}\eta_{t}\big)g_{m+1}(x_{t})-\eta_{t}^{2}10L_{\mathcal{F}}^{2}$.
\end{claim}
\begin{proof}
    The proof is by case distinction.
    
    \textbf{Case 1.} Suppose $g_{m+1}(x_{t})>0$. Using $\Vert x_{t}\Vert<R$
    it follows by Claim \ref{clm:velocity_ub_v1} that
    \[
	\Vert v_{t}\Vert\leq2\left(\alpha R+L_{\mathcal{F}}\right)=4L_{\mathcal{F}}.
    \]
    Using $g_{m+1}$ is concave and 1-smooth, $g_{m+1}(x_{t})>0$, $\nabla g_{m+1}(x_{t})=-x_{t}$
    and $\Vert x_{t}\Vert<R$, we have
    \begin{eqnarray*}
        g_{m+1}(x_{t+1})&\geq&g_{m+1}(x_{t})+\nabla g_{m+1}(x_{t})^{\T}(x_{t+1}-x_{t})-\frac{1}{2}\Vert x_{t+1}-x_{t}\Vert^{2}\\&\geq&-\eta_{t}R\Vert v_{t}\Vert-\frac{1}{2}\eta_{t}^{2}\Vert v_{t}\Vert^{2}\\&\geq&-\eta_{t}\cdot6L_{\mathcal{F}}R\\
        &\geq&-\eta_{t+1}\cdot7L_{\mathcal{F}}R,
    \end{eqnarray*}
    where we used \[
    \frac{1}{2}\eta_{t}16L_{\mathcal{F}}^{2}=\frac{8}{\sqrt{t+15}}L_{\mathcal{F}}R\leq2L_{\mathcal{F}}R.
    \] 
    
    \textbf{Case 2.} Suppose $g_{m+1}(x_{t})\leq0$, i.e., $\Vert x_{t}\Vert\geq R$.
    Using $\alpha^{2}R^{2}+\left(\alpha R+2L_{\mathcal{F}}\right)^{2}=10L_{\mathcal{F}}^2$, it follows by Claim~\ref{clm:velocity_ub_v1} that
	\[
		\frac{1}{2}\Vert v_{t}\Vert^{2}<-2\alpha^{2}g_{m+1}(x)+10L_{\mathcal{F}}^{2}.
	\]
    Combining $g_{m+1}$ is concave and 1-smooth, and $\nabla g_{m+1}(x_{t})^{\T}v_{t}\geq-\alpha g_{m+1}(x_{t})$
    yields
    \begin{eqnarray*}
        g_{m+1}(x_{t+1})&\geq&g_{m+1}(x_{t})+\nabla g_{m+1}(x_{t})^{\T}(x_{t+1}-x_{t})-\frac{1}{2}\Vert x_{t+1}-x_{t}\Vert^{2}\\&\geq&(1-\alpha\eta_{t})g_{m+1}(x_{t})-\frac{1}{2}\eta_{t}^{2}\Vert v_{t}\Vert^{2}\\&>&\left(1-\alpha\eta_{t}+2\alpha^{2}\eta_{t}^{2}\right)g_{m+1}(x_{t})-\eta_{t}^{2}10L_{\mathcal{F}}^{2}\\&\geq&\left(1-\frac{\alpha}{2}\eta_{t}\right)g_{m+1}(x_{t})-\eta_{t}^{2}10L_{\mathcal{F}}^{2},
    \end{eqnarray*}
    where the last inequality follows by: $-\eta_{t}\alpha+2\eta_{t}^{2}\alpha^{2}\leq-\eta_{t}\frac{\alpha}{2}$, which is implied by $\eta_{t}=\frac{1}{\alpha\sqrt{t+15}}$.
\end{proof}

\begin{lemma}[Main]\label{lem:gmp_all}
    Suppose the assertions in Claim~\ref{clm:velocity_ub_v1} hold for $c=4$.
    Given $\alpha=L_{\mathcal{F}}/R$, step sizes $\{\eta_{t}=\frac{1}{\alpha\sqrt{t+15}}\}_{t\geq1}$ and an arbitrary initial decision $x_{1}$ with $\Vert x_{1}\Vert<R$, then it holds that
    \begin{equation}\label{IH}
        g_{m+1}(x_{t})\geq-27\frac{R^{2}}{\sqrt{t}},\qquad\Vert x_{t}\Vert\leq4R,\qquad\Vert v_{t}\Vert\leq7L_{\mathcal{F}},\qquad\text{for all }t\geq1.
    \end{equation}
\end{lemma}
\begin{proof}
    The proof is by induction on $t\geq1$. 
    
    \textbf{Part I)} We show first that $g_{m+1}(x_{t+1})\geq-c_{0}\eta_{t+1}$, for some $c_0>0$. 
    The proof proceeds by case distinction.   
    
    \textbf{Case 1.} Suppose $g_{m+1}(x_{t})>0$, then by Claim \ref{clm:gmp} we have
	\[
	        g_{m+1}(x_{t+1})\geq-\eta_{t}\cdot6L_{\mathcal{F}}R,\qquad(\text{implying }c_0\geq6L_{\mathcal{F}}R).
	\]
    \textbf{Case 2.} Suppose $g_{m+1}(x_{t})\leq0$.
    Let $A:=10L_{\mathcal{F}}^{2}$, then by combining Claim \ref{clm:gmp} and the inductive hypothesis, we have
    \begin{eqnarray*}
        g_{m+1}(x_{t+1})&\geq&\left(1-\eta_{t}\frac{\alpha}{2}\right)g_{m+1}(x_{t})-\eta_{t}^{2}A\\&\geq&-\left(1-\eta_{t}\frac{\alpha}{2}\right)c_{0}\eta_{t}-\eta_{t}^{2}A\\&=&-c_{0}\eta_{t}-\left(A-\frac{\alpha}{2}c_{0}\right)\eta_{t}^{2}\\&=&-c_{0}\eta_{t+1}-c_{0}\eta_{t}+c_{0}\eta_{t+1}-\left(A-\frac{\alpha}{2}c_{0}\right)\eta_{t}^{2}\\&=&-c_{0}\eta_{t+1}+c_{0}\eta_{t}\left[-1+\frac{\eta_{t+1}}{\eta_{t}}-\eta_{t}\left(\frac{A}{c_{0}}-\frac{\alpha}{2}\right)\right].
    \end{eqnarray*}
    Since $c_0\eta_{t}>0$, it suffices to show that 
    \[
    -1+\frac{\eta_{t+1}}{\eta_{t}}-\eta_{t}\left(\frac{A}{c_{0}}-\frac{\alpha}{2}\right)\geq0\iff\frac{\alpha}{2}-\frac{\eta_{t}-\eta_{t+1}}{\eta_{t}^{2}}\geq\frac{A}{c_{0}}.
    \]
    The previous condition is equivalent to (using $\eta_{t}=\frac{1}{\alpha\sqrt{t+15}}$
    for $t\geq1$)
    \[
    \alpha\left[\frac{1}{2}-\frac{\sqrt{t+15}}{\sqrt{t+16}}\left[\sqrt{t+16}-\sqrt{t+15}\right]\right]\geq\frac{A}{c_{0}}.
    \]
    Straightforward checking shows that $\max_{t\geq16}\sqrt{\frac{t}{t+1}}\left(\sqrt{t+1}-\sqrt{t}\right)<0.12$
    and thus
    \[
       c_{0}\geq2.7\frac{A}{\alpha}=27L_{\mathcal{F}}R.
    \]
    Hence, for $c_{0}=27L_{\mathcal{F}}R$ it holds that $g_{m+1}(x_{t+1})\geq-c_{0}\eta_{t+1}$.
    We set $c_0$ to the maximum over the preceding two case, i.e.,
    \[
    c_{0}:=\max\left\{ 7L_{\mathcal{F}}R,\ 27L_{\mathcal{F}}R\right\},
    \]
    and obtain
    \[
	g_{m+1}(x_{t})\geq-c_{0}\eta_{t}=-\frac{27R^{2}}{\sqrt{t+15}}.
    \]
    
    \textbf{Part II)} We now show that $\Vert x_{t+1} \Vert\leq 4R$. 
    Combining \textbf{Part I)} and the definition of step size $\eta_t=\frac{1}{\alpha\sqrt{t+15}}$, we have
    \[
    \frac{1}{2}\left[R^{2}-\Vert x_{t+1}\Vert^{2}\right]=g_{m+1}(x_{t+1})\geq-c_{0}\eta_{t+1}\geq-c_{0}\eta_{1}=-\frac{c_{0}}{4\alpha}
    \]
    and thus
    \[
	\Vert x_{t+1}\Vert^{2}\leq R^{2}+\frac{c_{0}}{2\alpha}<15R^{2}<(4R)^{2}.
    \]
    
    \textbf{Part III)} By Claim~\ref{clm:velocity_ub_v1}, it follows that
    \[
    \Vert v_{t+1}\Vert\leq\frac{L_{\mathcal{F}}}{R}\Vert x_{t+1}\Vert+3L_{\mathcal{F}}<7L_{\mathcal{F}}.
    \]		
\end{proof}

\subsection{Concluding Remarks}

By Lemma~\ref{lem:gmp_all}, the decision sequence $\{x_t\}_{t\geq1}$ is attracted to the hypersphere $\mathcal{B}_{R}$ and always stays inside a slightly larger hypersphere $\mathcal{B}_{4R}$.

Then, by Lemma~\ref{lem:OGD} applied with $c=4$, $d=15$, $\alpha=L_{\mathcal{F}}/R$ and step size $\eta_t=1/(\alpha\sqrt{t+d})$ we obtain
\begin{eqnarray*}
    \mathrm{Regret}_{T}&\leq&\sqrt{15+1}\left[(4+3)^{2}+\frac{1}{2}(4+1)^{2}\right]L_{\mathcal{F}}R\sqrt{T}\\&=&246L_{\mathcal{F}}R\sqrt{T}.
\end{eqnarray*}

Moreover, by Lemma~\ref{lem:feasibility_convergence}, we have
$\mathcal{Z}_{d}=\frac{3}{2}\sqrt{17}\big[\frac{L_{\mathcal{G}}}{R}+\beta_{\mathcal{G}}\big]L_{\mathcal{F}}R$ and 
\[
c_{1}=\mathcal{V}_{\alpha}\left[2L_{\mathcal{G}}+\frac{\beta_{G}\mathcal{V}_{\alpha}}{\alpha}\right]+\mathcal{Z}_{d}\leq21\left[\frac{L_{\mathcal{G}}}{R}+3\beta_{G}\right]L_{\mathcal{F}}R.
\]
Hence, the convergence rate to the feasible $\mathcal{C}$ satisfies for every $t\geq1$ and $i\in\{1,\dots,m\}$ 
\[
g_{i}(x_{t})\geq-c_{1}\eta_{t}\geq-21\left[\frac{L_{\mathcal{G}}}{R}+3\beta_{G}\right]\frac{R^{2}}{\sqrt{t+15}}.
\]

\section{Proof of Theorem~\ref{thm:avrTVC}}\label{app:thm:avrTVC}

In this section, we consider an online optimization problem with adversarially generated time-varying constraints.
More precisely, at each time step $t$, the learner receives partial information on the current cost $f_t$ and feasible set $\mathcal{C}_t$, and seeks to minimize \eqref{eq:regret}.
To make this problem well posed, we restrict the environment such that each feasible set $\mathcal{C}_t$ is contained in $\mathcal{Q}_t$ (see Section~\ref{sec:Introduction}) and the rate of change between consecutive time-varying constraints \textit{decreases} over time.
We quantify a sufficient rate of decay in Assumption~\ref{as:TVC_decay}, which we restate below for the convenience of the reader.

\begin{assumption}[TVC Decay Rate]\label{as:stvc}
We assume that the adversarially generated sequence $\{g_{t}\}_{t\geq1}$ of time-varying constraints are such that for every $x\in\mathcal{B}_{4R}$ and all $t\geq1$, the following holds
$\Vert g_{t+1}(x)-g_{t}(x)\Vert_{\infty}\leq\frac{98}{t+16}\big[\frac{L_{\mathcal{G}}}{R}+3\beta_{\mathcal{G}}\big]R^{2}$.
\end{assumption}

We note that Assumption~\ref{as:stvc} essentially only requires $\Vert g_{t+1}(x)-g_{t}(x)\Vert_{\infty}\leq\mathcal{O}(1/t)$, as $R$ can be chosen large enough such that the bound is satisfied.
Of course, $R$ will appear in our regret and feasibility bounds, but it will not affect the dependence on $t$ or $T$ (up to constant factors).

We restate Theorem~\ref{thm:avrTVC} below for the convenience of the reader.

\begin{theorem}[Time-Varying Constraints]
Suppose the functions $\{f_{t},g_{t}\}_{t\geq1}$ satisfy Assumptions~\ref{as:fg},~\ref{as:ip} and \ref{as:stvc}.
Then, on input $R,L_{\mathcal{F}}>0$ and $x_{1}\in\mathcal{B}_{R}$, Algorithm~\ref{alg:cvv-pro} applied with $\alpha=L_{\mathcal{F}}/R$, augmented velocity polyhedron $V_{\alpha}^{\prime}(\cdot)$ and step sizes $\eta_{t}=\frac{1}{\alpha\sqrt{t+15}}$ guarantees the following for all $T\geq1$:
\vspace{-5pt}
\begin{itemize}[noitemsep, topsep=0pt, leftmargin=1pc]

\item[]\textbf{(regret)} \quad\,\,\, $\sum_{t=1}^{T}f_{t}(x_{t})-\min_{x^{\star}\in\mathcal{C}}\sum_{t=1}^{T}f_{t}(x^{\star})\leq 246L_{\mathcal{F}}R\sqrt{T}$;

\item[]\textbf{(feasibility)}\,\, $g_{t,i}(x_{t})\geq-265\left[\frac{L_{\mathcal{G}}}{R}+4\beta_{\mathcal{G}}\right]\frac{R^{2}}{\sqrt{t+15}}$,\quad for all $t\in\{1,\dots,T\}$ and $i\in\{1,\dots,m\}$;

\item[]\textbf{(attraction)}\,\, $g_{m+1}(x_{t})\geq-27\frac{R^{2}}{\sqrt{t+15}}$ for all $t\in\{1,\dots,T\}$.
\end{itemize}
\end{theorem}

\paragraph{Outline} This section is organized as follows. 
In Subsection~\ref{subsec:KeyGeopProp}, we introduce a key geometric property that allows us to generalize the standard online gradient descent analysis to the setting of time-varying constraints.
In Subsection~\ref{subsec:ProofOverviewThm:avrTCV}, we give an overview of our proof approach for Theorem~\ref{thm:avrTVC}.
In Subsection~\ref{subsec:slowlyTCV}, we present the analysis that quantifies the convergence rate to the feasible set for the setting of slowly time-varying constraints.
Finally, in Subsection~\ref{subsec:ATVC}, we give an important special case, slightly generalizing Lemma~\ref{lem:construct_average}, for which Assumption~\ref{as:stvc} is satisfied.

\subsection{Key Geometric Property}\label{subsec:KeyGeopProp}

Our regret analysis builds upon the following key geometric property that generalizes Lemma~\ref{lem:main} to time-varying constraints.
We show that for any subset $\mathcal{C}_T$ of the polyhedral intersection $\mathcal{Q}_T$, 
every decision $x\in\mathcal{C}_T$ satisfies the normal cone constraint $-r_{t}^{\T}(x - x_{t}) \leq 0$, for every pair $(x_t,r_t)$ in the decision sequence $\{(x_t,r_t)\}_{t=1}^{T}$ up to step $T$.
As a result, a similar argument as in \eqref{eq:OGD_regret} yields $\mathcal{O}(\sqrt{T})$ regret in the time-varying constraint setting.

\begin{lemma}[Polyhedral Intersection]\label{lem:time-varying-regret}
Let $\mathcal{C}_T$ be any subset of the polyhedral intersection $\mathcal{Q}_T$. Then, every decision $x\in\mathcal{C}_T$ satisfies the normal cone constraint $-r_{t}^{\T}(x - x_{t}) \leq 0$, $\forall t\in\{1,\dots,T\}$.
\end{lemma}

\begin{proof}
Using $\mathcal{S}_0=\mathbb{R}^n$, $\mathcal{C}_T$ is contained in $\cap_{t=1}^{T-1}\{x\in\mathbb{R}^n~|~G(x_t)^{\T}(x-x_t)\geq0\}$.
Since $x\in\mathcal{C}_T$, it follows by Lemma~\ref{lem:main} that $r_T^{\T}(x-x_T)\leq0$.
The proof proceeds by case distinction.
Let $t\in\{1,\dots,T-1\}$ be arbitrary.
Suppose $x_t\in\mathcal{C}_t$, then by Part 1 in the proof of Lemma~\ref{lem:main} we have $r_t=0$.
Suppose $x_t\not\in\mathcal{C}_t$, then $x\in\mathcal{C}_T$ implies $G(x_t)^{\T}(x-x_t)\geq0$ or equivalently $\nabla g_{t,i}(x_t)^{\T}(x-x_t)\geq0$ for all $i\in I(x_t)$.
Since $v_t\in V_{\alpha}(x_t)$, it follows that $v(x)=v_t+x-x_t\in V_{\alpha}(x_t)$.
Moreover, the vector $-r_t$ belongs to the normal cone $N_{V_{\alpha}(x_t)}(v_t)$, which implies $-r_t^{\T}(v-v_t)\leq0$ for all $v\in V_{\alpha}(x_t)$.
In particular, for $v(x)$ we have $-r_t^{\T}(x-x_t)\leq0$.
\end{proof}

\newpage
\subsection{Proof Overview of Theorem~\ref{thm:avrTVC}}\label{subsec:ProofOverviewThm:avrTCV}

By Assumption~\ref{as:fg}, the slowly time-varying constraints $g_{t,i}(x)$ are concave and $\beta_{\mathcal{G}}$-smooth such that $\Vert\nabla g_{t,i}(x)\Vert\leq L_{\mathcal{G}}$ for all $x\in\mathcal{B}_{4R}$, $t\geq1$ and $i\in\{1,\dots,m\}$.
By construction, see Lemma~\ref{lem:gmp_all}, $\eta_{t}=1/(\alpha\sqrt{t+15})$, $\alpha=L_{\mathcal{F}}/R$ and $\mathcal{V}_{\alpha}=7L_{\mathcal{F}}$ implies that $\eta_{t+1}\mathcal{V}_{\alpha}=7R/\sqrt{t+16}$. We note that Lemma~\ref{lem:gmp_all} still holds for time-varying constraints, which implies $\Vert x_t \Vert \leq 4R$ and $\Vert v_t \Vert\leq 7L_{\mathcal{F}}$.

Further, by Assumption~\ref{as:stvc} we have for every $x\in\mathcal{B}_{4R}$, $t\geq1$ and $i\in\{1,\dots,m\}$ that
\begin{equation}\label{eq:gtp_gt_VaR2}
\left|g_{t+1,i}(x)-g_{t,i}(x)\right|\leq\frac{98}{t+16}\left[\frac{L_{\mathcal{G}}}{R}+3\beta_{\mathcal{G}}\right]R^{2}=2\eta_{t+1}^{2}\left[\frac{L_{\mathcal{G}}}{R}+3\beta_{\mathcal{G}}\right]\mathcal{V}_{\alpha}^{2}.
\end{equation}
Then, applying the preceding inequality and using similar arguments as in Part 2) of Section~\ref{subsec:thm:main}, we give in Corollary~\ref{cor:TVC_ind_2} bounds on the slowly time-varying constraints $g_{t,i}(x)$ from below.
In particular, we show that
\[
g_{t+1,i}(x_{t+1})\geq(1-\alpha\eta_{t})g_{t,i}(x_{t})-\eta_{t}^{2}\left[2\frac{L_{\mathcal{G}}}{R}+7\beta_{\mathcal{G}}\right]\mathcal{V}_{\alpha}^{2}\qquad\text{for all } i\in I(x_{t}),
\]
and
\[
g_{t+1,i}(x_{t+1})\geq-\eta_{t+1}7\mathcal{V}_{\alpha}\left[L_{\mathcal{G}}+\frac{\beta_{\mathcal{G}}\mathcal{V}_{\alpha}}{4\alpha}\right]\qquad\text{for all } i\in \{1,\dots,m\}\backslash I(x_{t}).
\]
Using a similar inductive argument as in Lemma~\ref{lem:feasibility_convergence}, we show in Lemma~\ref{lem:AvrTVC} that in the setting of slowly time-varying constraints, the following feasibility convergence rate holds
\[
    g_{t,i}(x_{t})\geq-\left[265\frac{L_{\mathcal{G}}}{R}+927\beta_{\mathcal{G}}\right]\frac{R^{2}}{\sqrt{t+15}},\qquad\text{for all }t\in\{1,\dots,T\}\text{ and }i\in\{1,\dots,m\}.
\]
Then, the regret and the attraction to the feasible sets follow as in Theorem~\ref{thm:main_full}.

\subsection{Slowly Time-Varying Constraints}\label{subsec:slowlyTCV}

\begin{lemma}[Slowly TVC]\label{lem:AvrTVC}
    Suppose Assumption~\ref{as:fg} holds, $x_1\in\mathcal{B}_{R}$, $\alpha=L_{\mathcal{F}}/R$ and step sizes $\eta_t=1/(\alpha\sqrt{t+15})$.
    Then, for every $i\in\{1,\dots,m\}$ and $T\geq1$ we have
    \[
    g_{t,i}(x_{t})\geq-\left[265\frac{L_{\mathcal{G}}}{R}+927\beta_{\mathcal{G}}\right]\frac{R^{2}}{\sqrt{t+15}},\qquad\text{for all }t\in\{1,\dots,T\}\text{ and }i\in\{1,\dots,m\}.
    \]
\end{lemma}
\begin{proof}
    The proof is by induction on $t$. 
    We start with the base case $t=1$.
    The proof proceeds by case distinction.
    
    \textbf{Case 1.} Suppose $i\in \{1,\dots,m\}\backslash I(x_{1})$, i.e., $g_{1,i}(x_{1})>0$. 
    Then, by Corollary~\ref{cor:TVC_ind_2} Part ii) we have
    \[
        g_{2,i}(x_{2})\geq-\eta_{2}7\mathcal{V}_{\alpha}\left[L_{\mathcal{G}}+\frac{\beta_{\mathcal{G}}\mathcal{V}_{\alpha}}{4\alpha}\right]\geq-\eta_{2}\left[49\frac{L_{\mathcal{G}}}{R}+86\beta_{\mathcal{G}}\right]L_{\mathcal{F}}R.
    \] 
    \textbf{Case 2.} Suppose $i\in I(x_{1})$, i.e., $g_{1,i}(x_{1})\leq0$. 
    By combining $x_1\in\mathcal{B}_{R}$ and $g_{1,i}$ is concave $\beta_{\mathcal{G}}$-smooth, it follows for every $x\in\mathcal{C}_1\subseteq\mathcal{B}_{R}$ that
    \begin{eqnarray*}
        g_{1,i}(x_{1})&\geq&g_{1,i}(x)+\nabla g_{1,i}(x)^{T}(x_{1}-x)-\frac{\beta_{\mathcal{G}}}{2}\Vert x_{1}-x\Vert^{2}\\&\geq&-2L_{\mathcal{G}}R-2\beta_{\mathcal{G}}R^{2}\\&=&-\eta_{1}\left[8\frac{L_{\mathcal{G}}}{R}+8\beta_{\mathcal{G}}\right]L_{\mathcal{F}}R.
    \end{eqnarray*}
    Using $\eta_{t}=1/(\alpha\sqrt{t+15})$ and $\eta_{1}/\eta_{2}\leq\sqrt{2}$, it follows that 
    \[
    (1-\alpha\eta_{1})g_{1,i}(x_{1})\geq-\eta_{1}\left[\frac{L_{\mathcal{G}}}{R}+\beta_{\mathcal{G}}\right]6L_{\mathcal{F}}R\geq-\eta_{2}\left[9\frac{L_{\mathcal{G}}}{R}+9\beta_{\mathcal{G}}\right]L_{\mathcal{F}}R
    \]
    and
    \[
    \eta_{1}^{2}\left[2\frac{L_{\mathcal{G}}}{R}+7\beta_{\mathcal{G}}\right]\mathcal{V}_{\alpha}^{2}\leq\eta_{2}^{2}\left[4\frac{L_{\mathcal{G}}}{R}+14\beta_{\mathcal{G}}\right]\mathcal{V}_{\alpha}^{2}\leq\eta_{2}\left[49\frac{L_{\mathcal{G}}}{R}+172\beta_{\mathcal{G}}\right]L_{\mathcal{F}}R.
    \]
    Then, by Corollary~\ref{cor:TVC_ind_2} Part i) we have
    \begin{eqnarray*}
        g_{2,i}(x_{2})&\geq&(1-\alpha\eta_{1})g_{1,i}(x_{1})-\eta_{1}^{2}\left[2\frac{L_{\mathcal{G}}}{R}+7\beta_{\mathcal{G}}\right]\mathcal{V}_{\alpha}^{2}\\&\geq&-\eta_{2}\left[58\frac{L_{\mathcal{G}}}{R}+181\beta_{\mathcal{G}}\right]L_{F}R.
    \end{eqnarray*}
    Our inductive hypothesis is $g_{t,i}(x_{t})\geq-c_{2}\eta_{t}$ for all $i$.
    We now show that it holds for $t+1$. 
    
    \textbf{Case 1.} Suppose $i\in \{1,\dots,m\}\backslash I(x_{1})$, i.e., $g_{i}(x_{t})>0$.
    Then by Corollary~\ref{cor:TVC_ind_2} ii)
    \[
    g_{t+1,i}(x_{t+1})\geq-\eta_{t+1}7\mathcal{V}_{\alpha}\left[L_{\mathcal{G}}+\frac{\beta_{\mathcal{G}}\mathcal{V}_{\alpha}}{4\alpha}\right]\geq-\eta_{t+1}\left[49\frac{L_{\mathcal{G}}}{R}+86\beta_{\mathcal{G}}\right]L_{\mathcal{F}}R.
    \]

    \textbf{Case 2.} Suppose $i\in I(x_{t})$, i.e., $g_{i}(x_{t})\leq0$.
    Let $A=\left[2\frac{L_{\mathcal{G}}}{R}+7\beta_{\mathcal{G}}\right]\mathcal{V}_{\alpha}^{2}$.
    By combining Corollary~\ref{cor:TVC_ind_2} Part i), the inductive hypothesis and using similar arguments as in the proof of Lemma~\ref{lem:gmp_all} Case 2, yields
    \[
    g_{t+1,i}(x_{t+1})\geq-c_{2}\eta_{t+1},\quad\text{where}\quad c_{2}=2.7\frac{A}{\alpha}=\left[265\frac{L_{\mathcal{G}}}{R}+927\beta_{\mathcal{G}}\right]L_{\mathcal{F}}R.
    \]
    The feasibility convergence rate is then given by
    \[
    g_{t,i}(x_{t})\geq-\left[265\frac{L_{\mathcal{G}}}{R}+927\beta_{\mathcal{G}}\right]\frac{R^{2}}{\sqrt{t+15}}.
    \]
\end{proof}

\begin{corollary}\label{cor:TVC_ind_2} 
    Suppose Assumptions~\ref{as:fg} and Assumption~\ref{as:stvc} hold.
    Let $\alpha=L_{\mathcal{F}}/R$, $\mathcal{V}_{\alpha}=7L_{\mathcal{F}}$ and step sizes $\eta_{t}=1/(\alpha\sqrt{t+15})$.
    Then, for every $t\geq1$ we have
    
    i) $g_{t+1,i}(x_{t+1})\geq(1-\alpha\eta_{t})g_{t,i}(x_{t})-\eta_{t}^{2}\big[2\frac{L_{\mathcal{G}}}{R}+7\beta_{\mathcal{G}}\big]\mathcal{V}_{\alpha}^{2}$ for all $i\in I(x_{t})$; and
    
    ii) $g_{t+1,i}(x_{t+1})\geq-\eta_{t+1}7\mathcal{V}_{\alpha}\big[L_{\mathcal{G}}+\frac{\beta_{\mathcal{G}}\mathcal{V}_{\alpha}}{4\alpha}\big]$
    for all $i\in\{1,\dots,m\}\backslash I(x_{t})$.
\end{corollary}
\begin{proof}
    Combining Assumption~\ref{as:stvc} and \eqref{eq:gtp_gt_VaR2} gives
    \[
	g_{t+1,i}(x_{t+1})\geq g_{t,i}(x_{t+1})-2\eta_{t+1}^{2}\left[\frac{L_{\mathcal{G}}}{R}+3\beta_{\mathcal{G}}\right]\mathcal{V}_{\alpha}^{2}.
    \]
    Then, by Claim~\ref{clm:const_violation}, it follows for every $i\in I(x_{t})$ that
    \begin{eqnarray*}
        g_{t+1,i}(x_{t+1})&\geq&g_{t,i}(x_{t+1})-2\eta_{t+1}^{2}\left[\frac{L_{\mathcal{G}}}{R}+3\beta_{\mathcal{G}}\right]\mathcal{V}_{\alpha}^{2}.\\&\geq&(1-\alpha\eta_{t})g_{t,i}(x_{t})-\eta_{t}^{2}\frac{\mathcal{V}_{\alpha}^{2}\beta_{\mathcal{G}}}{2}-\eta_{t}^{2}\left[2\frac{L_{\mathcal{G}}}{R}+6\beta_{\mathcal{G}}\right]\mathcal{V}_{\alpha}^{2}\\&>&(1-\alpha\eta_{t})g_{t,i}(x_{t})-\eta_{t}^{2}\left[2\frac{L_{\mathcal{G}}}{R}+7\beta_{\mathcal{G}}\right]\mathcal{V}_{\alpha}^{2},
    \end{eqnarray*}
    and for every $i\in\{1,\dots,m\}\backslash I(x_{t})$ that
    \begin{eqnarray*}
        g_{t+1,i}(x_{t+1})&\geq&g_{t,i}(x_{t+1})-2\eta_{t+1}^{2}\left[\frac{L_{\mathcal{G}}}{R}+3\beta_{\mathcal{G}}\right]\mathcal{V}_{\alpha}^{2}\\&\geq&-\eta_{t+1}\mathcal{V}_{\alpha}\left[2L_{\mathcal{G}}+\frac{\beta_{\mathcal{G}}\mathcal{V}_{\alpha}}{4\alpha}\right]-\eta_{t+1}\mathcal{V}_{\alpha}\left[\frac{L_{\mathcal{G}}\mathcal{V}_{\alpha}}{2\alpha R}+\frac{3\beta_{\mathcal{G}}\mathcal{V}_{\alpha}}{2\alpha}\right]\\&\geq&-\eta_{t+1}7\mathcal{V}_{\alpha}\left[L_{\mathcal{G}}+\frac{\beta_{\mathcal{G}}\mathcal{V}_{\alpha}}{4\alpha}\right].
    \end{eqnarray*}
    where we used that $\alpha=L_{\mathcal{F}}/R$ and $\mathcal{V}_{\alpha}=7L_{\mathcal{F}}$ implies $\frac{L_{\mathcal{G}}\mathcal{V}_{\alpha}}{R\alpha}=7L_{\mathcal{G}}$.    
\end{proof}

\newpage
\subsection{Average Time-Varying Constraints}\label{subsec:ATVC}

An important special case where Assumption~\ref{as:stvc} is satisfied, is summarized in the following slightly more general version of Lemma~\ref{lem:construct_average}.

\begin{lemma}\label{lem:averageTVC}
Suppose the functions $\tilde{g}_{t,i}$ satisfy Assumption~\ref{as:fg} and in addition there is a decision $x_{t,i}\in\mathcal{B}_R$ such that $\left|\tg{t,i}(x_{t,i})\right|\leq\frac{1}{2}\big[\frac{L_{\mathcal{G}}}{R}+3\beta_{\mathcal{G}}\big]R^{2}$, for every $t\geq1$ and $i\in\{1,\dots,m\}$.
Then the following average time-varying constraints, satisfy Assumption~\ref{as:fg} and Assumption~\ref{as:stvc}:
\begin{equation}\label{eq:avrTVC}
	g_{t,i}(x):=\frac{1}{t}\sum_{\ell=1}^{t}\tg{\ell,i}(x)\in\mathbb{R}^{m}.
\end{equation}
\end{lemma}

The rest of this subsection is devoted to proving Lemma~\ref{lem:averageTVC}.
We achieve this in two steps.
We start by showing in Lemma~\ref{lem:linearity} that the average time-varying constraints satisfy Assumption~\ref{as:fg}, and then in Lemma~\ref{lem:hgit} we demonstrate that they also satisfy Assumption~\ref{as:stvc}.

\begin{lemma}\label{lem:linearity}
    Suppose $\tg{t,i}$ is concave $\beta_{\mathcal{G}}$-smooth such that $\lVert\nabla\tg{t,i}(x)\rVert\leq L_{\mathcal{G}}$ for all $x\in\mathcal{B}_{4R}$, $t\geq1$ and $i\in\{1,\dots,m\}$.
    Then, the average function
    \[
    g_{t,i}(x):=\frac{1}{t}\sum_{\ell=1}^{t}\tg{\ell,i}(x)
    \]
    is concave and $\beta_{\mathcal{G}}$-smooth and $\Vert\nabla g_{t,i}(x)\Vert\leq L_{\mathcal{G}}$ holds for all $x\in\mathcal{B}_{4R}$, $t\geq1$ and $i\in\{1,\dots,m\}$.
\end{lemma}
\begin{proof}
    By assumption, each $\tg{\ell,i}$ is concave and $\beta_{\mathcal{G}}$-smooth, which implies
    \[
    \tg{\ell,i}(x_{t+1})\geq\tg{\ell,i}(x_{t})+[\nabla\tg{\ell,i}(x_{t})]^{\T}[x_{t+1}-x_{t}]-\frac{\beta_{\mathcal{G}}}{2}\lVert x_{t+1}-x_{t}\rVert^{2}.
    \]
    Summing over all $\ell\in \{1,...,t\}$ yields
    \[
    \frac{1}{t}\sum_{\ell=1}^{t}\tg{\ell,i}(x_{t+1})\geq\frac{1}{t}\sum_{\ell=1}^{t}\tg{\ell,i}(x_{t})+\Big[\frac{1}{t}\sum_{\ell=1}^{t}\nabla\tg{\ell,i}(x_{t})\Big]^{\T}[x_{t+1}-x_{t}]-\frac{1}{t}\sum_{\ell=1}^{t}\frac{\beta_{\mathcal{G}}}{2}\lVert x_{t+1}-x_{t}\rVert^{2},
    \]
    since $\frac{1}{t}\sum_{\ell=1}^{t}\nabla\tg{\ell,i}(x)=\nabla g_{t,i}(x)$, which is equivalent to
    \[
    g_{t,i}(x_{t+1})\geq g_{t,i}(x_{t})+[\nabla g_{t,i}(x)]^{\T}[x_{t+1}-x_{t}]-\frac{\beta_{\mathcal{G}}}{2}\lVert x_{t+1}-x_{t}\rVert_{2}^{2}.
    \]
    Hence, $g_{t,i}$ is concave and $\beta_{\mathcal{G}}$-smooth.		
    
    Moreover, since $\lVert\nabla\tg{t,i}(x)\rVert\leq L_{\mathcal{G}}$
    for all $x\in\mathcal{B}_{4R}$, we have
    \[
    \Vert\nabla g_{t,i}(x)\Vert=\left\Vert \frac{1}{t}\sum_{\ell=1}^{t}\nabla\tg{\ell,i}(x)\right\Vert \leq\frac{1}{t}\sum_{\ell=1}^{t}\Vert\nabla\tg{\ell,i}(x)\Vert\leq L_{\mathcal{G}}.
    \]
\end{proof}

We show next that the average time-varying constraints satisfy Assumption~\ref{as:stvc}.

\begin{lemma}[Average TVC]\label{lem:hgit}
    Suppose $\tg{t,i}$ is concave $\beta_{\mathcal{G}}$-smooth such that $\lVert\nabla\tg{t,i}(x)\rVert\leq L_{\mathcal{G}}$ for all $x\in\mathcal{B}_{4R}$, $t\geq1$ and $i\in\{1,\dots,m\}$.
    Further, suppose for every $t\geq1$ and $i\in\{1,\dots,m\}$, there exists a decision $x_{t,i}\in\mathcal{B}_{R}$ such that
    \begin{equation}\label{eq:tgti}
    \left|\tg{t,i}(x_{t,i})\right|\leq\frac{1}{2}\left[\frac{L_{\mathcal{G}}}{R}+3\beta_{\mathcal{G}}\right]R^{2}.
    \end{equation}
    Then, for $\alpha=L_{\mathcal{F}}/R$, step sizes $\eta_t=1/(\alpha\sqrt{t+15})$ and $\mathcal{V}_{\alpha}=7L_{\mathcal{F}}$, it holds for every $x\in\mathcal{B}_{4R}$ that
    \[
    \left|g_{t+1,i}(x)-g_{t,i}(x)\right|\leq2\eta_{t+1}^{2}\left[\frac{L_{\mathcal{G}}}{R}+3\beta_{\mathcal{G}}\right]\mathcal{V}_{\alpha}^{2}.
    \]
\end{lemma}
\begin{proof}
    Using the inequality $\frac{1}{t+1}\leq\frac{17}{2}\frac{1}{t+16}$ for every $t\geq1$ and $\eta_{t+1}^2=1/(\alpha^2(t+16))$, it follows by construction that
    \begin{eqnarray}\label{eq:gbttp}
        \left|g_{t+1,i}(x)-g_{t,i}(x)\right|&=&\left|\frac{1}{t+1}\tg{t+1,i}(x)+\frac{t}{t+1}g_{t,i}(x)-g_{t,i}(x)\right|\nonumber\\&=&\frac{1}{t+1}\left|\tg{t+1,i}(x)-g_{t,i}(x)\right|\nonumber\\&=&\frac{1}{t+1}\frac{1}{t}\left|\sum_{\ell=1}^{t}\tg{t+1,i}(x)-\tg{\ell,i}(x)\right|\nonumber\\&\leq&\eta_{t+1}^{2}\frac{17}{2}\alpha^{2}\cdot\frac{1}{t}\sum_{\ell=1}^{t}\left|\tg{t+1,i}(x)-\tg{\ell,i}(x)\right|.
    \end{eqnarray}
    By triangle inequality $\left|\tg{t+1,i}(x)-\tg{\ell,i}(x)\right|\leq\left|\tg{t+1,i}(x)\right|+\left|\tg{\ell,i}(x)\right|$
    and thus it suffices to bound the term $|\tg{t,i}(x)|$
    for every $t\geq1$, $i\in\{1,\dots,m\}$ and $x\in\mathcal{B}_{4R}$.
    
    By assumption, $x\in\mathcal{B}_{4R}$ and there is $x_{t,i}\in\mathcal{B}_{R}$ satisfying inequality~\eqref{eq:tgti}. Further, $\tg{t,i}$ is concave, which implies
    \[
    \tg{t,i}(x)-\tg{t,i}(x_{t,i})\leq[\nabla\tg{t,i}(x_{t,i})]^{\T}[x-x_{t,i}]\leq5L_{\mathcal{G}}R
    \]
    and the fact that $\tg{t,i}$ is concave $\beta_{\mathcal{G}}$-smooth yields
    \begin{eqnarray*}
        \tg{t,i}(x)-\tg{t,i}(x_{t,i})&\geq&[\nabla\tg{t,i}(x_{t,i})]^{\T}[x-x_{t,i}]-\frac{\beta_{\mathcal{G}}}{2}\Vert x_{t,i}-x\Vert^{2}\\&\geq&-5\left[\frac{L_{\mathcal{G}}}{R}+3\beta_{\mathcal{G}}\right]R^{2}.
    \end{eqnarray*}
    Further, by combining $\big|\tg{t,i}(x)-\tg{t,i}(x_{t,i})\big| \leq 5\big[\frac{L_{\mathcal{G}}}{R}+3\beta_{\mathcal{G}}\big]R^{2}$, triangle inequality and assumption~\eqref{eq:tgti}, we obtain for every $x\in\mathcal{B}_{4R}$ that
    \begin{eqnarray*}
        \left|\tg{t,i}(x)\right|&=&\left|\tg{t,i}(x)-\tg{t,i}(x_{t,i})+\tg{t,i}(x_{t,i})\right|\\&\leq&\left|\tg{t,i}(x)-\tg{t,i}(x_{t,i})\right|+\left|\tg{t,i}(x_{t,i})\right|\\&\leq&\frac{11}{2}\left[\frac{L_{\mathcal{G}}}{R}+3\beta_{\mathcal{G}}\right]R^{2}.
    \end{eqnarray*}
    The statement follows by combining $\alpha=L_{\mathcal{F}}/R$, $\mathcal{V}_{\alpha}=7L_{\mathcal{F}}$, \eqref{eq:gbttp} and
    \begin{eqnarray*}
        \left|g_{t+1,i}(x)-g_{t,i}(x)\right|&\leq&\eta_{t+1}^{2}\frac{17}{2}\alpha^{2}\cdot\frac{1}{t}\sum_{\ell=1}^{t}\left|\tg{t+1,i}(x)-\tg{\ell,i}(x)\right|\\&\leq&\eta_{t+1}^{2}\left[\frac{L_{\mathcal{G}}}{R}+3\beta_{\mathcal{G}}\right]\frac{11}{2}\cdot17\alpha^{2}R^{2}\\&<&2\eta_{t+1}^{2}\left[\frac{L_{\mathcal{G}}}{R}+3\beta_{\mathcal{G}}\right]\mathcal{V}_{\alpha}^{2}.
    \end{eqnarray*}
\end{proof}

\end{document}